\def\arxivfmt{1} 

\if\arxivfmt{1}
\documentclass[12pt]{article}
\else
\documentclass{article}
\fi

\usepackage{microtype}
\usepackage{graphicx}
\usepackage{adjustbox}
\usepackage{booktabs} 

\usepackage{url}
\usepackage{mathrsfs}
\usepackage{amsmath, latexsym}
\usepackage{amscd}
\usepackage{amsbsy}
\usepackage{amssymb}
\usepackage{amsthm}
\usepackage{todonotes}
\usepackage{comment} 
\usepackage{hyperref}
\usepackage[capitalize]{cleveref}
\usepackage{nicefrac}
\usepackage[sort,numbers]{natbib}
\usepackage{graphicx}
\usepackage{tikz}
\usetikzlibrary{positioning,arrows}
\usepackage{amssymb,amsmath,amsthm}

\theoremstyle{plain}
\newtheorem{theorem}{Theorem}

\newtheorem{lemma}{Lemma}
\newtheorem{proposition}{Proposition}
\theoremstyle{definition}
\newtheorem{assumption}{Assumption}
\newtheorem{definition}{Definition}
\newtheorem{remark}[theorem]{Remark}

\usepackage{empheq}
\usepackage{shortcuts}
\usetikzlibrary{fit}

\usepackage{hyperref}

\usepackage{subcaption} 

\usepackage[Algorithm,ruled]{algorithm}
\usepackage{algorithmic}

\Crefname{assumption}{Assumption}{Assumptions}

\newcommand{\norm}[1]{\left\lVert#1\right\rVert}

\newcommand{\primalfeasible}{F(w)}

\newcommand{\defeq}{\vcentcolon=}
\newcommand{\Wset}{\mathcal{B}}
\newcommand{\Wsetctrlv}{\tilde{\mathcal{B}}}

\newcommand{\ie}{\emph{i.e.}}
\newcommand{\eg}{\emph{e.g.}}

\newcommand{\uset}{\mathcal{U}}
\newcommand{\wstat}{w}
\newcommand{\wmarg}{\tilde{w}}
\newcommand{\Wpi}{\beta}
\newcommand{\Wpimarg}{g}
\newcommand{\staterewards}{\Phi}
\newcommand{\abnd}{l}
\newcommand{\bbnd}{m}
\newcommand{\Pstat}{p^{\infty}}
\newcommand{\Pstatcompactb}{\stackrel{\infty_b}{p}}
\newcommand{\st}{\text{s.t. }}
\newcommand{\polval}{R}
\newcommand{\ts}{\textstyle}

\newcommand{\Ss}{ \mathcal S  }
\newcommand{\Uu}{ \mathcal U  }
\newcommand{\nS}{\vert \mathcal S \vert }

\newcommand{\nA}{\vert \mathcal A \vert }

\usepackage{wrapfig}

\ifcase\arxivfmt 0
\usepackage{icml2020}
\or 1
\usepackage{fullpage}
\usepackage{authblk}
\else
\fi

\if\arxivfmt 1
\usepackage{fullpage}
\usepackage{authblk}
\fi

\ifcase\arxivfmt 0 
\icmltitlerunning{Confounding-Robust Policy Evaluation in RL}
\else 

\fi

\begin{document}

\ifcase\arxivfmt 0 
\twocolumn[
\icmltitle{Confounding-Robust Policy Evaluation \\in Infinite-Horizon Reinforcement Learning}
\icmlsetsymbol{equal}{*}

\begin{icmlauthorlist}
	\icmlauthor{Aeiau Zzzz}{equal,to}
	\icmlauthor{Bauiu C.~Yyyy}{equal,to,goo}
	\icmlauthor{Cieua Vvvvv}{goo}
	\icmlauthor{Iaesut Saoeu}{ed}
	\icmlauthor{Fiuea Rrrr}{to}
	\icmlauthor{Tateu H.~Yasehe}{ed,to,goo}
	\icmlauthor{Aaoeu Iasoh}{goo}
	\icmlauthor{Buiui Eueu}{ed}
	\icmlauthor{Aeuia Zzzz}{ed}
	\icmlauthor{Bieea C.~Yyyy}{to,goo}
	\icmlauthor{Teoau Xxxx}{ed}
	\icmlauthor{Eee Pppp}{ed}
\end{icmlauthorlist}

\icmlaffiliation{to}{Department of Computation, University of Torontoland, Torontoland, Canada}
\icmlaffiliation{goo}{Googol ShallowMind, New London, Michigan, USA}
\icmlaffiliation{ed}{School of Computation, University of Edenborrow, Edenborrow, United Kingdom}

\icmlcorrespondingauthor{Cieua Vvvvv}{c.vvvvv@googol.com}
\icmlcorrespondingauthor{Eee Pppp}{ep@eden.co.uk}

\icmlkeywords{Machine Learning, ICML}

\vskip 0.3in
]



\printAffiliationsAndNotice{\icmlEqualContribution} 
\or
\title{Confounding-Robust Policy Evaluation \\in Infinite-Horizon Reinforcement Learning}
\author[1,2]{Nathan Kallus\thanks{kallus@cornell.edu}}
\author[1,2]{Angela Zhou\thanks{az434@cornell.edu}}
\affil[1]{School of Operations Research and Information Engineering, Cornell University}
\affil[2]{Cornell Tech, Cornell University}
\renewcommand\Authands{ and }
\date{}
\maketitle\vspace{-\baselineskip}
\large
\else
\title{Confounding-Robust Policy Evaluation \\in Infinite-Horizon Reinforcement Learning}
\author[1,2]{Nathan Kallus\thanks{kallus@cornell.edu}}
\author[1,2]{Angela Zhou\thanks{az434@cornell.edu}}
\affil[1]{School of Operations Research and Information Engineering, Cornell University}
\affil[2]{Cornell Tech, Cornell University}
\renewcommand\Authands{ and }
\date{}
\maketitle\vspace{-\baselineskip}
\large
\maketitle
\fi

\begin{abstract}
	\ifcase\arxivfmt0
\or 1
\large
\else 
	\fi
	\normalsize
Off-policy evaluation of sequential decision policies from observational data is necessary in applications of batch reinforcement learning such as education and healthcare. In such settings, however, unobserved variables confound observed actions, rendering exact evaluation of new policies impossible, \ie, unidentifiable. We develop a robust approach that estimates sharp bounds on the (unidentifiable) value of a given policy in an infinite-horizon problem given data from another policy with unobserved confounding, subject to a sensitivity model. We consider stationary or baseline unobserved confounding
and compute bounds by optimizing over
the set of all stationary state-occupancy ratios that agree with a new partially identified estimating equation and the sensitivity model.
We prove convergence to the sharp bounds as we collect more confounded data. Although checking set membership is a linear program, the support function is given by a difficult nonconvex optimization problem. We 
develop approximations based on nonconvex projected gradient descent and demonstrate the resulting bounds empirically.
\end{abstract}

\section{Introduction}
\vspace{-5pt}
Evaluation of sequential decision-making policies under uncertainty is a fundamental problem for learning sequential decision policies from observational data, as is necessarily the case in application areas such as education and healthcare \citet{jl16,pb2016,pss2001}. However, with a few exceptions, the literature on off-policy evaluation in reinforcement learning (RL) assumes (implicitly or otherwise) the \textit{absence} of unobserved confounders, auxiliary state information that affects both the policy that generated the original data as well as transitions to the next state. Precisely in the same important domains where observational off-policy evaluation 
is \textit{necessary} due to the cost of or ethical constraints on experimentation, such as in healthcare \citep{raghu2017deep,prasad2017reinforcement} or operations, it is \textit{also} generally the case that \textit{unobserved confounders} are present. 
This contributes to fundamental challenges for advancing reinforcement learning in observational settings \citep{gottesman2019guidelines}. 

In this work, we study partial identification in RL off-policy evaluation under unobserved confounding, focusing specifically on the \textit{infinite-horizon} setting. Recognizing that policy value cannot actually be point-identified from confounded observational data, we propose instead to compute the sharpest bounds on policy value that can be supported by the data and any assumptions on confounding. This can then support credible conclusions about policy value from the data and can ensure safety in downstream policy learning.

Recent advancements \citep{hallak2017consistent,gelada2019off,liu2018breaking,kallus2019efficiently} improve variance reduction of \emph{unconfounded} off-policy evaluation by estimating density ratios on the \textit{stationary} occupancy distribution. But this assumes unconfounded data. Other advances \citep{kallus2018confounding} tackle partial identification of policy values from confounded data but in the logged bandit setting (single decision point) rather than the RL setting (many or infinite decision points). Our work can be framed as appropriately combining these perspectives.

Our contributions are as follows: we establish a partially identified estimating equation that allows for the estimation of sharp bounds. We provide tractable approximations of the resulting difficult non-convex program based on non-convex first order methods. We then demonstrate the approach on a gridworld task with unobserved confounding.

\vspace{-5pt}
\section{Problem setup} 
\vspace{-5pt}
We assume data is generated from an infinite-horizon MDP with an augmented state space: $\Ss$ is the space of the \textit{observed} portion of the state and $\Uu$ is the space of the \textit{unobserved} (confounding) portion of the state. 
We assume the standard decision protocol for MDPs on the full-information state space $\Ss \times \Uu$: at each decision epoch, the system occupies state $s_t,u_t$, the decision-maker receives a reward $\staterewards(s_t)$ for being in state $s_t$ and chooses an action, $a_t$, from allowable actions.
Then the system transitions to the next state on $\Ss \times \mathcal U$, with the (unknown) transition probability $p(s',u' \mid s,u,a)$. 
The full-information MDP is represented by the tuple $M= (\mathcal S \times \mathcal U, \mathcal A, P, \Phi)$. We let $\mathcal{H}_t = \{ (s_0, u_0, a_0), \dots, (s_t, u_t, a_t) \}$ denote the (inaccessible) full-information history up to time $t$. A policy $\pi(a \mid s,u)$ is an assignment to the probability of taking action $a$ in state $(s,u)$. 
For any policy,
	the underlying dynamics are Markovian under full observation of states and transitions:
	$s_t \indep \mathcal{H}_{t-2} \mid (s_{t-1}, u_{t-1}), a_{t-1} $.

In the off-policy evaluation setting, we consider the case where the observational data are generated under an unknown \textit{behavior policy} $\pi_b$, while we are interested in evaluating the (known) \emph{evaluation policy} $\pi_e$, which only depends on the observed state, $\pi_e(a\mid s,u)=\pi_e(a\mid s)$. Both policies are assumed stationary (time invariant). The observational dataset does not have full information and comprises solely of  observed states and actions, that is, $(s_0, a_0), \dots, (s_t, a_t)$.\footnote{Our model differs from typical POMDPs \citep{kaelbling1998planning}, since rewards are a function of observed state, as we clarify in the related work, \Cref{sec-related-work}.}
Thus, since the action also depends on the unobserved state $u_t$, we have that transition to next states are \emph{confounded} by $u_t$.

Notationally, we reserve $s,u$ (respectively, $s', u'$) for the random variables representing state (respectively, next state) and we refer to realized observed state values (respectively, next observed state values) using $j$ (respectively, $k$). 
We assume that $\Ss$ is a discrete state space, while $\Uu$ may be general.

We next discuss regularity conditions on the MDP structure
which ensure ergodicity and that the limiting state-action occupancy frequencies exist. We assume that the Markov chain induced by $\pi_e$ and any $\pi_b$ is a positive Harris chain, so the stationary distribution exists.
\begin{assumption}[Ergodic MDP]\label{asn-recurrent}
The MDP $M$ is ergodic: 
the Markov chains induced by $\pi_b$ and $\pi_e$ is 
positive Harris recurrent.
\end{assumption}

	In this work, we focus on the infinite-horizon setting. 
	Let $p_{\pi}^{(t)}(s)$ be the distribution of state $s_t$ when executing policy $\pi$, starting from initial state $s_0$ drawn from an initial distribution over states. 
    Then the average \textit{state-action-next-state} visitation distribution exists, and under \Cref{asn-recurrent} the (long-run average) \emph{value} of a stationary policy $\pi_e$ is given by an expectation with respect to the marginalized state visitation distribution: 
\begin{equation*}\ts  p_{\pi_e}^{\infty}(s,u,a,s',u') = \lim_{T\to \infty} \frac1T{\sum_{t=0}^T  p_{\pi_e}^{(t)}(s,u,a,s',u')  },
\qquad \ts \polval_{e} = \E_{s \sim \Pstat_{\pi_e}}[ \staterewards(s)].
\end{equation*}
We similarly define the marginalized total-, unobserved-, and observed-state occupancy distributions as $p_{\pi}^{\infty}(s,u)$, $p_{\pi}^{\infty}(u)$, and $p_{\pi}^{\infty}(s)$, given by appropriately marginalizing the above.
Notice we assumed that the reward only depends on the \emph{observed} state\footnote{This does not preclude, however, dependence on action: if we are given observed-state-action reward function $\staterewards'(s,a)$, we may simply define $\staterewards(s)=\sum_a\pi_e(a\mid s)\staterewards'(s,a)$, since $\pi_e(a\mid s)$ is assumed known. Then $\polval_{e}$ gives $\pi_e$'s value with respect to the given observed-state-action reward function.}.

Notationally, $\E$ denotes taking expectations over the joint stationary occupancy distribution of the behavior policy, where self-evident. We denote $\Pstat_e, \Pstat_b$ for visitation distributions induced under $\pi_e, \pi_b$. Since at times it is useful to distinguish between expectation over the marginalized occupancy distribution $\Pstat_b(s,a,s')$, and total expectation over full-information transitions $\Pstat_b(s,u,a,s',u')$, we include additional subscripts on the expectation whenever this is clarifying.

If we were able to actually run the MDP using the policy $\pi_e$, which is only a function of $s$, the dynamics would be Markovian with the marginalized transition probabilities:
		$$ \ts p(k\mid j,a)  \defeq \sum_{u',u''}  p(k,u'' \mid j,u', a) \Pstat_{e} (u'\mid j) $$
Note that $p(k\mid j,a)$ is \emph{not} identifiable from the observational data collected under $\pi_b$. We analogously define (partially) marginalized transition probabilities $ p(k \mid j, u', a) $.
\vspace{-5pt}
\section{Off-policy evaluation under unobserved confounding}\label{sec-ope-unobs-conf} 
\vspace{-5pt}
In the following, we first discuss a \textit{population} viewpoint, computing expectations with respect to the true marginalized stationary occupancy distribution $\Pstat_b(s,a,s')$ (identifiable from the data). We discuss finite-sample considerations in \cref{sec-cons-sharpness}.

Given sample trajectories generated from $\pi_b$,
the goal of off-policy evaluation is to estimate $\polval_e$, the value of a known (observed-state-dependent) evaluation policy $\pi_e(a \mid s)$. The full-information stationary density ratio is $\wstat(s,u)$. If $\wstat(s,u)$ \textit{were} known, we could use it to estimate policy value under $\pi_e$ using samples from $\pi_b$ by a simple density ratio argument:
$$ \ts \wstat(s,u)  =
 \frac{p_{e}^{(\infty)} (s,u)  }{ p_{b}^{(\infty)} (s,u)    }, \qquad  \polval_{e} = \E
 [\wstat(s,u)  \staterewards(s)]$$
From observational data, we are only able to estimate the \emph{marginalized} behavior policy, $\pi_b(a \mid s) = \E[ \pi_b(a \mid s,u) \mid s], $ which is insufficient for identifying the policy value or the true marginalized transition probabilities.

\subsection{Model restrictions on unobserved confounding } 
\vspace{-5pt}
To make progress, we introduce restrictions on the underlying dynamics of the unobserved confounder, $u$, under which we will conduct our evaluation. In particular, we will seek to compute the range of all values of $R_e$ that match our data, encapsulated in $\Pstat_b(s,a,s')$, and the following structural assumptions.

\begin{assumption}[Memoryless Unobserved Confounding]\label{asn-no-time-dependent-confounding}
\begin{equation*} \textstyle 
\frac{p_{e}^{(\infty)} (s,u)  }{ p_{b}^{(\infty)} (s,u)}
=\frac{p_{e}^{(\infty)} (s,u')  }{ p_{b}^{(\infty)} (s,u')}
\qquad \forall s\in\Ss,u, u' \in \uset
\end{equation*}
\end{assumption}
\begin{lemma}\label{lemma-suff-asn-no-time-dependent}
	\Cref{asn-no-time-dependent-confounding} holds if the MDP transitions satisfy $$\ts
	~p(j,u' \mid k,u, a) = ~p(j,\tilde u' \mid k,u, a), \; \forall u',\tilde u' \in \uset.$$
\end{lemma}
\ifcase\arxivfmt 0
\begin{figure}[t!]
	\centering
	\begin{tikzpicture}[%
	>=latex',node distance=2cm, minimum height=0.75cm, minimum width=0.75cm,
	state/.style={draw, shape=circle, draw=black, fill=green!2, line width=0.5pt},
	u/.style={draw, shape=circle, draw=black, fill=purple!2, line width=0.5pt},
	action/.style={draw, shape=rectangle, draw=red, fill=red!2, line width=0.5pt},
	reward/.style={draw, shape=rectangle, draw=blue, fill=blue!2, line width=0.5pt},
	scale = 0.8, transform shape
	]
	\node[state] (S0) at (0,0) {$s_0$};
	\node[action,right of=S0] (A0) {$a_0$};
	\node[u,below  of=A0] (u0) {$u_0$};
	\node[state,right of=A0] (S1) {$s_1$};
	\node[action,right of=S1] (A1) {$a_1$};
	\node[u,below  of=A1] (u1) {$u_1$};
	\node[state,right of=A1] (S2) {$s_2$};
	\draw[->] (S0) -- (A0);
	\draw[->] (u0) -- (A0);
	\draw[->] (u0) -- (S1);
	\draw[->] (S0) edge[bend left=30] (S1);
	\draw[->] (A0) edge[bend left=30] (S1);
	\draw[->] (S1) -- (A1);
	\draw[->] (u1) -- (A1);
	\draw[->] (u1) -- (S2);
	\draw[->] (S1) edge[bend left=30] (S2);
	\draw[->] (A1) edge[bend left=30] (S2);
	\end{tikzpicture}\\[0.5em]
	\begin{tikzpicture}[%
	>=latex',node distance=2cm, minimum height=0.75cm, minimum width=0.75cm,
	state/.style={draw, shape=circle, draw=black, fill=green!2, line width=0.5pt},
	u/.style={draw, shape=rectangle, rounded corners=.2cm, draw=black, fill=purple!2, line width=0.5pt},
	action/.style={draw, shape=rectangle, draw=red, fill=red!2, line width=0.5pt},
	reward/.style={draw, shape=rectangle, draw=blue, fill=blue!2, line width=0.5pt},
	scale = 0.8, transform shape
	]
	\node[state] (S0) at (0,0) {$s_0$};
	\node[action,right of=S0] (A0) {$a_0$};
	\node[u,below  of=S1] (u0) {$u_0=u_1=u_2=\cdots$};
	\node[state,right of=A0] (S1) {$s_1$};
	\node[action,right of=S1] (A1) {$a_1$};
	\node[state,right of=A1] (S2) {$s_2$};
	\draw[->] (S0) -- (A0);
	\draw[->] (u0) -- (A0);
	\draw[->] (u0) -- (S1);
	\draw[->] (S0) edge[bend left=30] (S1);
	\draw[->] (A0) edge[bend left=30] (S1);
	\draw[->] (S1) -- (A1);
	\draw[->] (u0) -- (A1);
	\draw[->] (u0) -- (S2);
	\draw[->] (S1) edge[bend left=30] (S2);
	\draw[->] (A1) edge[bend left=30] (S2);
	\end{tikzpicture}
	\caption{Two causal models satisfying \Cref{asn-no-time-dependent-confounding}.}\label{fig-dag-rep} 
\end{figure}
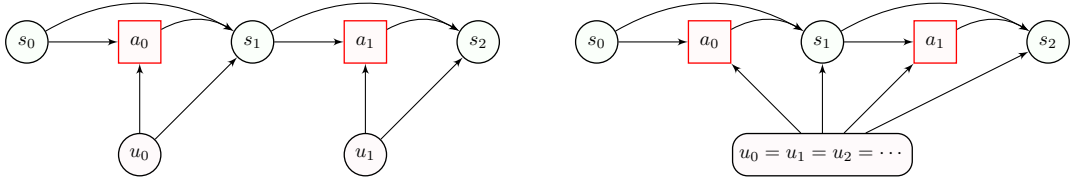
\or 1
\begin{figure}[t!]
	\centering
	\begin{tikzpicture}[%
	>=latex',node distance=2cm, minimum height=0.75cm, minimum width=0.75cm,
	state/.style={draw, shape=circle, draw=black, fill=green!2, line width=0.5pt},
	u/.style={draw, shape=circle, draw=black, fill=purple!2, line width=0.5pt},
	action/.style={draw, shape=rectangle, draw=red, fill=red!2, line width=0.5pt},
	reward/.style={draw, shape=rectangle, draw=blue, fill=blue!2, line width=0.5pt},
	scale = 0.8, transform shape
	]
	\node[state] (S0) at (0,0) {$s_0$};
	\node[action,right of=S0] (A0) {$a_0$};
	\node[u,below  of=A0] (u0) {$u_0$};
	\node[state,right of=A0] (S1) {$s_1$};
	\node[action,right of=S1] (A1) {$a_1$};
	\node[u,below  of=A1] (u1) {$u_1$};
	\node[state,right of=A1] (S2) {$s_2$};
	\draw[->] (S0) -- (A0);
	\draw[->] (u0) -- (A0);
	\draw[->] (u0) -- (S1);
	\draw[->] (S0) edge[bend left=30] (S1);
	\draw[->] (A0) edge[bend left=30] (S1);
	\draw[->] (S1) -- (A1);
	\draw[->] (u1) -- (A1);
	\draw[->] (u1) -- (S2);
	\draw[->] (S1) edge[bend left=30] (S2);
	\draw[->] (A1) edge[bend left=30] (S2);
	\end{tikzpicture}\\[0.5em]
	\begin{tikzpicture}[%
	>=latex',node distance=2cm, minimum height=0.75cm, minimum width=0.75cm,
	state/.style={draw, shape=circle, draw=black, fill=green!2, line width=0.5pt},
	u/.style={draw, shape=rectangle, rounded corners=.2cm, draw=black, fill=purple!2, line width=0.5pt},
	action/.style={draw, shape=rectangle, draw=red, fill=red!2, line width=0.5pt},
	reward/.style={draw, shape=rectangle, draw=blue, fill=blue!2, line width=0.5pt},
	scale = 0.8, transform shape
	]
	\node[state] (S0) at (0,0) {$s_0$};
	\node[action,right of=S0] (A0) {$a_0$};
	\node[u,below  of=S1] (u0) {$u_0=u_1=u_2=\cdots$};
	\node[state,right of=A0] (S1) {$s_1$};
	\node[action,right of=S1] (A1) {$a_1$};
	\node[state,right of=A1] (S2) {$s_2$};
	\draw[->] (S0) -- (A0);
	\draw[->] (u0) -- (A0);
	\draw[->] (u0) -- (S1);
	\draw[->] (S0) edge[bend left=30] (S1);
	\draw[->] (A0) edge[bend left=30] (S1);
	\draw[->] (S1) -- (A1);
	\draw[->] (u0) -- (A1);
	\draw[->] (u0) -- (S2);
	\draw[->] (S1) edge[bend left=30] (S2);
	\draw[->] (A1) edge[bend left=30] (S2);
	\end{tikzpicture}
	\caption{Two causal models satisfying \Cref{asn-no-time-dependent-confounding}.}\label{fig-dag-rep} 
\end{figure}
\else
\begin{figure}[t!]
	\begin{tikzpicture}[%
	>=latex',node distance=2cm, minimum height=0.75cm, minimum width=0.75cm,
	state/.style={draw, shape=circle, draw=black, fill=green!2, line width=0.5pt},
	u/.style={draw, shape=circle, draw=black, fill=purple!2, line width=0.5pt},
	action/.style={draw, shape=rectangle, draw=red, fill=red!2, line width=0.5pt},
	reward/.style={draw, shape=rectangle, draw=blue, fill=blue!2, line width=0.5pt},
	scale = 0.75, transform shape
	]
	\node[state] (S0) at (0,0) {$s_0$};
	\node[action,right of=S0] (A0) {$a_0$};
	\node[u,below  of=A0] (u0) {$u_0$};
	\node[state,right of=A0] (S1) {$s_1$};
	\node[action,right of=S1] (A1) {$a_1$};
	\node[u,below  of=A1] (u1) {$u_1$};
	\node[state,right of=A1] (S2) {$s_2$};
	\draw[->] (S0) -- (A0);
	\draw[->] (u0) -- (A0);
	\draw[->] (u0) -- (S1);
	\draw[->] (S0) edge[bend left=30] (S1);
	\draw[->] (A0) edge[bend left=30] (S1);
	\draw[->] (S1) -- (A1);
	\draw[->] (u1) -- (A1);
	\draw[->] (u1) -- (S2);
	\draw[->] (S1) edge[bend left=30] (S2);
	\draw[->] (A1) edge[bend left=30] (S2);
	\end{tikzpicture}$\;\;\qquad\;$\begin{tikzpicture}[%
	>=latex',node distance=2cm, minimum height=0.75cm, minimum width=0.75cm,
	state/.style={draw, shape=circle, draw=black, fill=green!2, line width=0.5pt},
	u/.style={draw, shape=rectangle, rounded corners=.2cm, draw=black, fill=purple!2, line width=0.5pt},
	action/.style={draw, shape=rectangle, draw=red, fill=red!2, line width=0.5pt},
	reward/.style={draw, shape=rectangle, draw=blue, fill=blue!2, line width=0.5pt},
	scale = 0.75, transform shape
	]
	\node[state] (S0) at (0,0) {$s_0$};
	\node[action,right of=S0] (A0) {$a_0$};
	\node[u,below  of=S1] (u0) {$u_0=u_1=u_2=\cdots$};
	\node[state,right of=A0] (S1) {$s_1$};
	\node[action,right of=S1] (A1) {$a_1$};
	\node[state,right of=A1] (S2) {$s_2$};
	\draw[->] (S0) -- (A0);
	\draw[->] (u0) -- (A0);
	\draw[->] (u0) -- (S1);
	\draw[->] (S0) edge[bend left=30] (S1);
	\draw[->] (A0) edge[bend left=30] (S1);
	\draw[->] (S1) -- (A1);
	\draw[->] (u0) -- (A1);
	\draw[->] (u0) -- (S2);
	\draw[->] (S1) edge[bend left=30] (S2);
	\draw[->] (A1) edge[bend left=30] (S2);
	\end{tikzpicture}
	\caption{Two causal models satisfying \Cref{asn-no-time-dependent-confounding}.}\label{fig-dag-rep} 
\end{figure}
\fi

\Cref{asn-no-time-dependent-confounding} essentially requires no time-varying confounders, \ie, confounders that are influenced by past actions. This assumption may appear strong but it is necessary: if confounders could be time-varying and the dependence on them may be arbitrary, we may need to be ``exponentially conservative'' in accounting for them (or even ``infinitely conservative'' in the infinite-horizon case)\footnote{We discuss this in \cref{sec-apx-fh} of the appendix in considering a finite-horizon off-policy estimate.}. 

For example, \Cref{asn-no-time-dependent-confounding} holds if there is just some baseline unobserved confounder for each individual, or if 
the unobserved confounder is exogenously drawn at each timestep (which satisfies the sufficient condition of \cref{lemma-suff-asn-no-time-dependent}).
These examples are shown in \Cref{fig-dag-rep}. A combination of a baseline unobserved confounder as well as exogenous confounders at each timestep is also allowed. Such MDPs would satisfy \Cref{asn-no-time-dependent-confounding} for \emph{any} two policies. In healthcare, baseline confounders such as socio-economic status or risk/toxicity preferences, confound actions by affecting access to doctors and treatments, or by affecting choices between intensive or conservative treatments. For example, \cite{yang2018sensitivity} identifies psychosocial factors and underlying comorbidities as baseline (unobserved) confounders in assessing sequential HIV treatment. \Cref{lemma-suff-asn-no-time-dependent} also holds if a confounded behavior policy arises from agents optimizing upon exogenous realizations of private information, as in econometric discrete choice models. 

Under \Cref{asn-no-time-dependent-confounding}, we simply define $w(s)=w(s,u)$ as it does not depend on $u$.
Note that $w(s)$ remains unidentifiable even under \Cref{asn-no-time-dependent-confounding}.
\vspace{-5pt}
\subsection{Sensitivity model}\label{sec-ope-sensitivity}
\vspace{-5pt}
Next, we introduce a sensitivity model to control the level of assumed dependence of the behavior policy on the unobserved confounders. Sensitivity analysis allows a practitioner to assess how conclusions might change for reasonable ranges of $\Gamma$. Following \citep{kallus2018confounding,aronowlee12} we phrase this as lower and upper bounds on the (unknown) inverse behavior policy,
 $\abnd(a\mid s),\bbnd(a \mid s)$:\footnote{Our approach immediately generalizes to any linearly representable ambiguity set.}
\begin{equation}\label{eq:boundset}
\ts\Wpi( a \mid s, u) := \pi_{b}(a  \mid s, u )^{-1}, \qquad \abnd(a \mid s) \leq \Wpi(a \mid s, u) \leq \bbnd(a \mid s)\quad\forall a,s,u. 
\end{equation}
The set $\Wset$ consists of all functions $\Wpi(a \mid s, u)$ that satisfy \cref{eq:boundset}. This ambiguity set is motivated by a sensitivity model used in causal inference, which restricts how far propensities can vary pointwise from the nominal propensities \citep{tan} and which has also been used in the logged bandit setting \citep{kallus2018confounding}.
Given a sensitivity parameter $\Gamma\geq1$ that controls the amounts of allowed confounding, the marginal sensitivity model posits the following odds-ratio restriction: 
	\begin{equation}\label{oddsratio}
\ts \Gamma^{-1}\leq
\frac{ (1- \pi_b(a \mid s)  ) \pi_b(a \mid s,u) }{ \pi_b(a \mid s)   (1- \pi_b(a \mid s,u) )  } 
\leq \Gamma,
\quad \forall a,s,u.
\end{equation}
\cref{oddsratio} is equivalent to saying that \cref{eq:boundset} holds with 
	\ifcase\arxivfmt0
\begin{align*}l(a\mid s)&={\Gamma/(\pi_b(a \mid s))+1-\Gamma},\\m(a\mid s)&={1/(\Gamma\pi_b(a \mid s))+1-1/\Gamma}.\end{align*}
\or 1
\begin{align*}l(a\mid s)&={\Gamma/(\pi_b(a \mid s))+1-\Gamma},\\m(a\mid s)&={1/(\Gamma\pi_b(a \mid s))+1-1/\Gamma}.\end{align*}
\else 
\begin{align*}l(a\mid s)={\Gamma/(\pi_b(a \mid s))+1-\Gamma},\qquad m(a\mid s)={1/(\Gamma\pi_b(a \mid s))+1-1/\Gamma}.\end{align*}
\fi
Lastly, $\Wpi$ functions which are themselves valid inverse probability distributions must satisfy the next-state conditional constraints: 
\begin{equation}\label{eqn-conditional-ctrlvar} 
\ts \E_{s,u, a, s' \sim p_b^\infty}[ \frac{ \mathbb{I}[a=a']}{\pi_b(a' \mid s,u)} \mid s'=k ] =\Pstat_b(k\mid a') 
~~~ \forall k,a'
\end{equation}
We let $\Wsetctrlv$ denote the set of all functions $\Wpi(a \mid s, u)$ that satisfy \emph{both} \cref{eq:boundset,eqn-conditional-ctrlvar}. 

\vspace{-5pt}
\subsection{The partially identified set}\label{sec-pi-set}
\vspace{-5pt}
Given the above restrictions, we can define the set of partially identified evaluation policy values.
To do so, suppose we are given a target behavior policy $\pi_e$, the observed stationary distribution $p_b^\infty(s,a,s')$, and bounds $\abnd(a \mid s),\bbnd(a \mid s)$ on $\beta$. We are then concerned with what $w$ could be, given the degrees of freedom that remain. So, we define the following set for what values $w$ can take:\footnote{Note that, as defined, $\Theta$ is a set of functions of $(s,u)$ but because we enforce \cref{asn-no-time-dependent-confounding}, all members are constant with respect to $u$ for each $s$; we therefore often implicitly refer to it as a set of functions of $s$ alone.}
$$
\Theta=\braces{
	\frac{p_{e}^{(\infty)}(s,u)}{p_{b}^{(\infty)} (s,u)}
	~~:~~
	\parbox{10cm}{
		$M$ is an MDP,
		
		$M$ satisfies \cref{asn-recurrent,asn-no-time-dependent-confounding} with respect to $\pi_b$ and the given $\pi_e$,
		
		$\pi_b(a\mid s,u)$ is a stationary policy with $\Wpi\in\Wset$
		and $p_b^\infty(s,a,s')$ as given
		
	}
}
$$

We are then interested in determining the largest and smallest that $R_e$ can be. That is, we are interested in
\begin{equation}\label{eq:Re}
\underline R_e=\inf_{w\in\Theta}\E[w(s)\Phi(s)],\quad\overline R_e=\sup_{w\in\Theta}\E[w(s)\Phi(s)].
\end{equation}
Notice that this is equivalent to computing the \emph{support function} of $\Theta$ at $-\Phi$ and $\Phi$ with respect to the $L_2$ inner product defined by $p_b^\infty(s)$, $\ip{f}{g}=\E[f(s)g(s)]=\sum_jp_b^\infty(j)f(j)g(j)$.
The support function of a set $\mathcal S$ is $\psi(v)=\sup_{s\in\mathcal S}\ip v s$ \citep{rockafellar1970convex}.

\section{Characterizing the partially identified set}

In this section we derive a linear program to check membership in $\Theta$ for any given $w$.
\vspace{-5pt}
\paragraph{The partially identified estimating equation}
We begin by showing that $w$ is uniquely characterized by an estimating equation characterizing its stationarity, but where some parts of the equation are not actually known.

\begin{lemma}\label{prop-lumpable-infhorizon}
Suppose \Cref{asn-recurrent,asn-no-time-dependent-confounding} hold. Then $\wstat(s)=\frac{p_{e}^{(\infty)}(s,u)}{p_{b}^{(\infty)} (s,u)}$ $\forall s,u$  if and only if
\begin{align}
&\E
[ \pi_e(a \mid s) \wstat (s) \Wpi(a \mid s,u)  \mid s' = k ]= \wstat(k)~~\forall k,
 \label{eqn-esteqn-w}\\
 &\E[w(s)] = 1.\label{eqn-esteqn-w-1}
\end{align}
\end{lemma}
The forward implication of \Cref{prop-lumpable-infhorizon} follows from Theorem 1 of \citep{liu2018breaking} applied to the state variable $(s,u)$ after recognizing that $w(s,u)$ only depends on $s$ under \Cref{asn-no-time-dependent-confounding} and marginalizing out $u'$. The backward implications of \Cref{prop-lumpable-infhorizon} follows from the recurrence of the aggregated MDP obtained from the transform $(s,u)\mapsto s$ \citep{kemeny1960finite}. A complete proof appears in the appendix.

Fortunately, \cref{eqn-esteqn-w,eqn-esteqn-w-1} exactly characterize $w$. Unfortunately, \cref{eqn-esteqn-w} involves two unknowns: $\beta(a\mid s,u)$ and the distribution $p_b^\infty(s,u,a,s')$ with respect to which the expectation is taken. In that sense, the estimated equation is only partially identified. Nonetheless, this allows to make progress toward a tractable characterization of $\Theta$.
\vspace{-5pt}
\paragraph{Marginalization} 
We next show that when optimizing over $\Wpi \in \Wset$, 
the sensitivity model can be reparametrized with respect to \textit{marginal weights} $g_k(a\mid j)$ (in the following, $j,a,k$ are generic indices into $\Ss, \mathcal A , \Ss$, respectively):
\begin{align*}\textstyle\Wpimarg_k(a \mid j) 
&\textstyle \defeq \sum_{u'}  
\frac{	p^{(\infty)}_{b} ( j,u', a  \mid k  )}{\Pstat_b ( j, a  \mid k  )	}
\Wpi (a\mid j,u')= \left(	\sum_{u'}\pi_b (a \mid j,u') 	\frac{ \Pstat_b(u'\mid j ) p(k\mid j,u', a)}{p(k\mid j, a) }\right)^{-1}
\end{align*}
Note that the values of the $g_k(a\mid j)$ weights are \textit{not} equivalent to the confounded $\pi_b(a \mid s)^{-1}$:
the difference is exactly the \textit{variability} in the underlying full-information transition probabilities  $p(k\mid j,u', a)$.
We will show that $g_k(a\mid j) \in \tilde\Wset$ satisfies the following constraints, where \cref{eqn-conditional-moment} corresponds to \cref{eqn-conditional-ctrlvar}: 
\begin{align}
\textstyle \abnd(a \mid j) &\leq \Wpimarg_k(a\mid j) \leq \bbnd(a \mid j) , &&
\forall j,a,k \nonumber
\\
\textstyle	\Pstat_b (k \mid a) &= \sum_{j} \Pstat_b (j,  a,k) g_k(a\mid j),&&
\forall k,a\label{eqn-conditional-moment} 
\end{align}
Reparametrization with respect to $g_k(a\mid j)$ follows from an optimization argument, recognizing the symmetry of optimizing a function of unknown realizations of $u$ with respect to an unknown conditional visitation density. Crucially, reparametrization improves the scaling of the number of nonconvex bilinear variables from the number of samples or trajectories, $O(NT)$, to $O(\nS^2 \nA)$.

Unlike sensitivity models in causal inference, it is possible that the partial identification set is empty, $\Theta = \emptyset$,
even if its associated sensitivity model $\tilde\Wset$ is nonempty in the space of weights.
 In \cref{apx-obs-implications} of the appendix, we explain this further by studying the state-action-state polytope \cite{mannor2005empirical}, the set of all valid visitation distributions achievable by some policy. The next result summarizes that imposing constraints on the \textit{marginalized} joint distribution $\Pstat_b(s,a,s')$ is insufficient to ensure the full (unobserved) joint distribution corresponds to a
 valid MDP stationary distribution.

\begin{proposition}[]\label{prop-sensitivity-sharpness} 
	The implementable implications of the marginalized state-action and marginalized state-action-state polytopes are:
$
\ts	\Pstat_b (k \mid a) = \sum_{j} \Pstat_b (j,  a,k) g_k(a\mid j), \forall k \in \Ss, a \in \mathcal A
$.
\end{proposition}
 \Cref{prop-sensitivity-sharpness} justifies our restrictions on $\tilde\Wset$. However, it also implies that further imposing constraints on $\tilde\Wset$
\textit{cannot} ensure compatibility of 
$g_k(a\mid j)$ for the observed $\Pstat_b(s,a,s')$ (and therefore $\Theta$), where compatibility is the requirement that $\Pstat_b(s,a,s')$ is stationary for $g_k(a \mid j)$. 
\vspace{-5pt}
\paragraph{Feasibility Linear Program} 
We next show that $w\in\Theta$ can be expressed using the linear program $\primalfeasible$ that
minimizes the L1 norm of residuals of the estimating equation of \Cref{prop-lumpable-infhorizon}, for a given $w$, over the sensitivity model $\Wpimarg \in \tilde\Wset$\footnote{Linearity of $F(w)$ also holds if an instrument function is used to convert the conditional moment equality to an unconditional one, as in Eqn. 10 \citep{liu2018breaking}, and as we use in \Cref{section-opt-pgd,lem-matrix-inversion-body}.}:
\begin{equation}
\textstyle\primalfeasible\defeq
	\underset{g \in \tilde\Wset}{\min}\;
\sum_k 
\abs{  \sum_{j,a}
p^{(\infty)}_{b} ( j, a  \mid k  )\wstat(j) \pi_e (a \mid j) \Wpimarg_k(a\mid j)  -w(k) 
}.
\end{equation}

\begin{proposition}[Feasibility Linear Program]\label{lemma-weight-reparametrization}
	\begin{equation}\label{eqn:pi-inf-horz-reparametrized-pop}
\textstyle	w \in  \Theta \iff \primalfeasible \leq 0, \; \E[w(s)] = 1
	\end{equation}
\end{proposition}
In relation to \cref{prop-sensitivity-sharpness}, the analysis for \cref{eqn:pi-inf-horz-reparametrized-pop} shows that it is exactly the partially identified estimating equation, \Cref{prop-lumpable-infhorizon}, which enforces compatibility such that
combining the restrictions on $\tilde\Wset$ and \Cref{prop-lumpable-infhorizon}
verifies membership of $w$ in $\Theta$. A consequence of \Cref{lemma-weight-reparametrization} is sharpness of the partially identified interval $[\underline R_e,\overline R_e]$; each point in the interval corresponds to some policy value. 
	\begin{theorem}[Sharpness]\label{thm-sharpness}$ \{\E[w(s)\Phi(s)]:w\in\Theta\}=[\underline R_e,\overline R_e].$
\end{theorem}

\section{Optimizing over the partially identified set}\label{sec-opt-PI} 

 \cref{eqn:pi-inf-horz-reparametrized-pop} suggests computing $\underline R_e,\overline R_e$ by solving \begin{equation}\label{eqn-opt-Fw}
\ts  \inf/\sup~\{ \E[w(s)\Phi(s)] \colon F(w) \leq 0,~\E[w(s)]=1 \}.
 \end{equation}
The restriction $F(w) = 0$ implicitly encodes an optimization over $g$,
 resulting in a hard nonconvex bilevel optimization problem. We first show that the structure of the feasibility program admits reformulation as a disjunctive program.

\begin{proposition}\label{prop-disjunctive} 
	\cref{eqn-opt-Fw} can be reformulated as a disjunctive program (hence a finite linear program).
\end{proposition}
However, the size of the resulting program is super-exponential in the cardinality of the state space, hence not practical. In \cref{apx-empirics} we also discuss the computational intractability of lifted SDP relaxations, which incur $O(\nS^9)$ iteration complexity. Therefore, for small state spaces, we suggest to solve \cref{eqn-opt-Fw} directly via Gurobi. 

\begin{figure}
\begin{subfigure}[t!]{.5\textwidth}
		\caption{Algorithm 1: Nonconvex 
			nonconvex-projected gradient descent}
		\label{alg:pgd}
		\begin{algorithmic}
			\STATE {\bfseries Input:} step size $\eta_0$, exponent $\kappa\in(0,1]$, 
			
		initial iterate $g_0$,  number of iterations $N$
			\FOR{$k=0,\dots,N-1$}
			\STATE $\eta_k \gets \eta_0 t^{-k}$
			\STATE
			 $
			w^*_{g_k} \in
			 \underset{\E[w(s)] = 1}
			{\arg\min} \{\norm{A(g_k) w }_1   \} $		
			
				$\tilde{g}_{k} \in \underset{g \in \tilde\Wset }{\arg\min} \{ \norm{ g - g_k}_1
				\colon A(g) w^*_{g_k} = 0 \} $
			\STATE $g_{k+1} \gets  \op{Proj}_{\Wset} (g_k + \eta_t   \nabla_{\tilde{g}_{k}} (\varphi^T\tilde A(g)^{-1}v))$
			\ENDFOR  \qquad Return $g_{k}$ with the best loss. 
		\end{algorithmic}
\end{subfigure}\begin{minipage}[t!]{0.23\textwidth}
		\centering
		\begin{adjustbox}{max width=0.95\textwidth}
			\begin{tikzpicture}[
			>=latex',node distance=2cm, minimum height=0.75cm, minimum width=0.75cm,
			state/.style={draw, shape=circle, draw=black, fill=green!2, line width=0.5pt},
			u/.style={draw, shape=circle, draw=black, fill=purple!2, line width=0.5pt},
			action/.style={draw, shape=rectangle, draw=red, fill=red!2, line width=0.5pt},
			reward/.style={draw, shape=rectangle, draw=blue, fill=blue!2, line width=0.5pt}
			]
			\node[u] (u0) at (0,0) {$u_0$};
			\node[u,below  of=u0] (u1) {$u_1$};
			\node[draw,inner sep=1mm
			,label=below
			:,fit=(u0) (u1)] (s1) {$s_1$} ;
			\node[u] (u0s2) at (3,0) {$u_0$};
			\node[u,below  of=u0s2] (u1s2) {$u_1$}; 
			\node[draw,inner sep=1mm
			,label=below
			:,fit=(u0s2) (u1s2)](s2) {$s_2$};
			\draw[->] (u0) -- (s2);
			\draw[->] (u1) -- (s2);
			\draw[->] (u0s2) -- (s1);
			\draw[->] (u1s2) -- (s1);
			\draw[->] (u0.north) arc (0:240:0.5) (s1);
			\draw[->] (u1.south) arc (0:-240:0.5) (s1);
			\draw[->] (u0s2.north) arc (180:-75:0.5) (s2);
			\draw[->] (u1s2.south) arc (-180:75:0.5) (s2);
			\end{tikzpicture}
		\end{adjustbox}
		\caption{Confounded random walk.}\label{fig-conf-random-walk} 
	\end{minipage}\hspace{0.1cm}\begin{minipage}[t!]{0.25\textwidth}\centering
		\includegraphics[width=\textwidth]{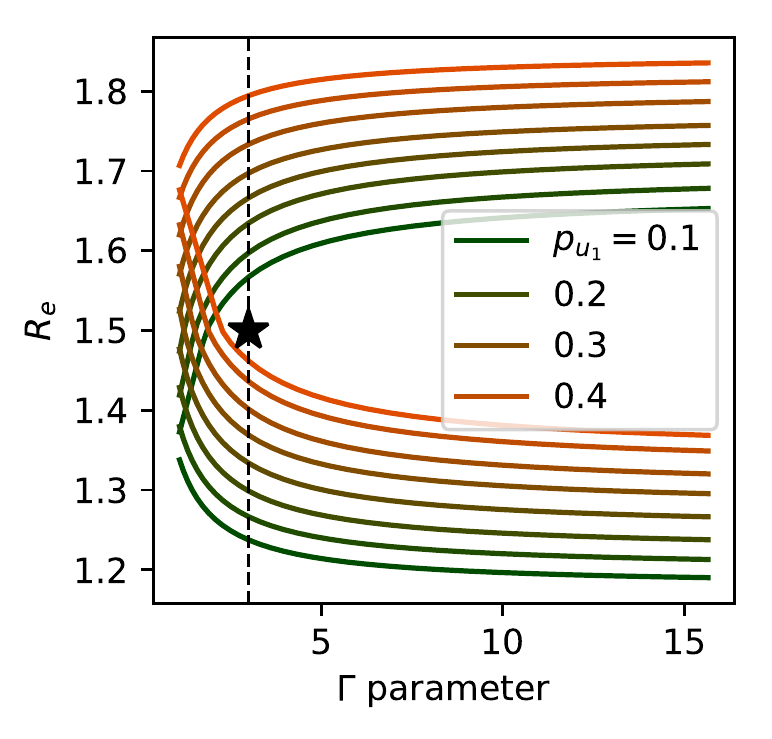}
		\caption{Varying transition models on $(s,u)$. $\star$ is true policy value.}\label{fig-conf-random-walk-varying-transitions} 
	\end{minipage}
\end{figure}
\vspace{-5pt}
\paragraph{ Nonconvex nonconvex-projected gradient method}\label{section-opt-pgd} 


We next develop a more practical optimization approach based on non-convex first-order methods.
First we restate
the estimating equation \cref{eqn-esteqn-w} for a fixed $g$ as a matrix system.
To evaluate expectations on the unconditional joint distribution, we introduce instrument functions
$\phi_s,\phi_{s'} \in \mathbb{R}^{\nS \times 1 }$, random (row) vectors which are one-hot indicators for the state random variable $s, s'$ taking on each value, $ \phi_s=  \begin{bmatrix}
\mathbb{I} [s=0] & \dots &  \mathbb{I} [s=\nS]
\end{bmatrix}$. Let $A(g) = \E[\phi_{s'}( {\pi_e(a\mid s)g_{s'}(a\mid s)} \phi_s - \phi_{s'})^\top]$ and $b_s=p_b^{\infty}(s)$. 
Let $\psi$ be the set of $\Wpimarg \in \tilde\Wset$ that admit a feasible solution to the estimating equation for some $\wstat \in \Theta$: 
\begin{equation}\label{eqn-g-feas} 
\psi\defeq \textstyle\{  g \in \tilde\Wset \colon  
\;\exists \; w\geq 0
\text{ s.t. } 
A(g) w = 0, b^\top w= 1  \}
\end{equation}
Define $\tilde A(g)$ by replacing the last row of $A(g)$ by $b$ and let $v=(0,\dots,0,1)\in\R{\abs{\Ss}}$. 
\begin{proposition}\label{lem-matrix-inversion-body}
If $g\in\psi$ then $\tilde A(g)$ is invertible. Moreover, $
\Theta=\{\tilde A(g)^{-1}v~~:~~g\in\psi\}.
$
\end{proposition}

\Cref{lem-matrix-inversion-body} suggests computing 
$\underline R_e,\,\overline R_e$ by solving \begin{equation}\label{eqn-opt-nonconvex}
\ts  \inf/\sup~\{ \varphi^T\tilde A(g)^{-1}v ~~\colon~~ g\in\psi \},
 \end{equation}
 where $\varphi_s=\Phi(s)p_b^\infty(s)$.
This optimization problem has both a non-convex objective and a non-convex feasible set, but it has small size.
As a way to approximate $\underline R_e,\,\overline R_e$, we propose a gradient descent approach to solving \cref{eqn-opt-nonconvex} in Algorithm 1. Since the feasible set is itself non-convex, we use an approximate projection that corrects each $g$ iterate to a feasible point but may not be a projection. This is based on taking alternating projection steps on $
w^*_{g_k} \in \arg\min \{\norm{A(g_k) w }_1 \colon \E[w(s)] = 1 \} $ and 
$\tilde{g}_{k} \in \arg\min \{ \norm{ g - g_k}_1 \colon A(g) w^*_{g_k} = 0, g \in \tilde\Wset \} $; each of these is a linear program with $\nS$ or $\nS^2 \nA$ many variables, respectively.

\remark Since the tabular setting is a special case of linear function approximation for $w$, our approach directly handles the case where $w = \theta^\top s$ is a linear function of the state, but further requires well-specification. See \cref{apx-linearfunctionapprox} for detail and discussion of additional challenges.

\vspace{-5pt}
\section{Consistency}\label{sec-cons-sharpness} 
\vspace{-5pt}
The above analysis considered the population setting,
but in practice we use the empirical state-action occupancy distribution, $\hat{p}^\infty_b(s,a,s')$. Define $\hat{\overline{\polval}}_e, \hat{\underline{\polval}}_e$ as the corresponding values when we solve \cref{eq:Re} with this estimate in place of $\Pstat_b(s,a,s')$. We establish consistency of the estimated bounds.
\begin{theorem}[Consistency]\label{thm-consistency} 
If $\hat{p}^\infty_b(s,a,s')\to\Pstat_b(s,a,s')$, then
$\ts\hat{\overline{\polval}}_e \to  {\overline{\polval}}_e, \,    \hat{\underline{\polval}}_e \to  {\underline{\polval}}_e$.
\end{theorem}
Since the empirical distributions satisfy $\hat{p}^\infty_b(s,a,s')\to_p\Pstat_b(s,a,s')$ \citep[see, \eg,][]{mannor2005empirical}, \cref{thm-consistency} and the continuous mapping theorem would then together imply that $\ts\hat{\overline{\polval}}_e \to_p  {\overline{\polval}}_e, \,    \hat{\underline{\polval}}_e \to_p  {\underline{\polval}}_e$. Since the perturbation of ${p}^\infty_b$ to $\hat{p}^\infty_b$ 
introduces perturbations to the \textit{constraint matrix} of the LP,
to prove this result we leverage a general stability analysis due to \citep[Theorem 1]{robinson1975stability}.


\begin{figure*}
\centering
	\begin{minipage}[t]{0.3\textwidth}
	\includegraphics[width=\textwidth]{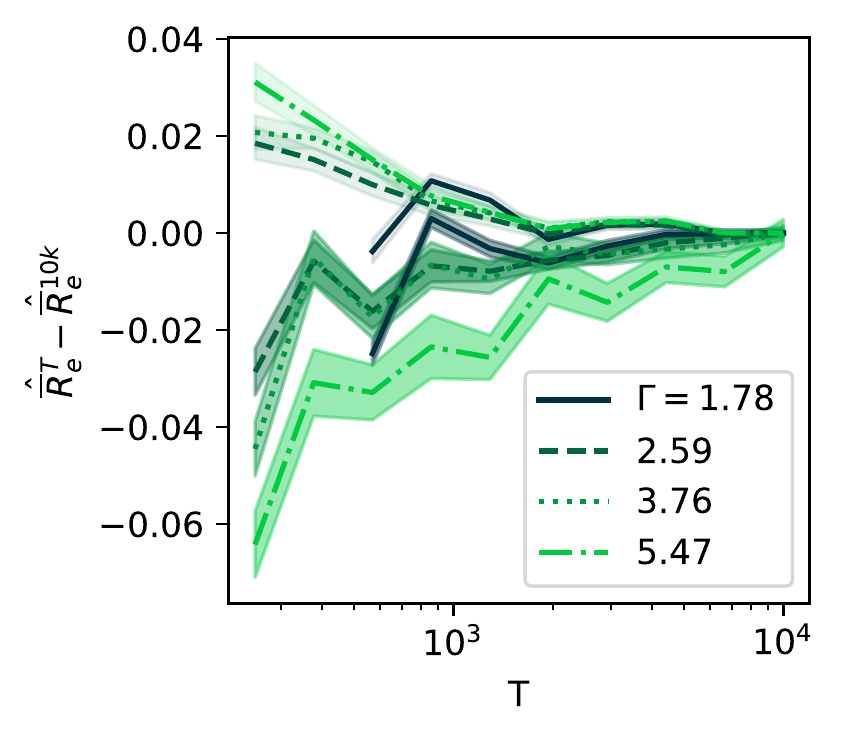} 
		\caption{Statistical 
			\\
			consistency}
	\label{fig-stat-consistency} 
\end{minipage}\begin{minipage}[t]{0.35\textwidth}\centering
		\captionsetup{width=.8\linewidth}   
		\includegraphics[width=\textwidth]{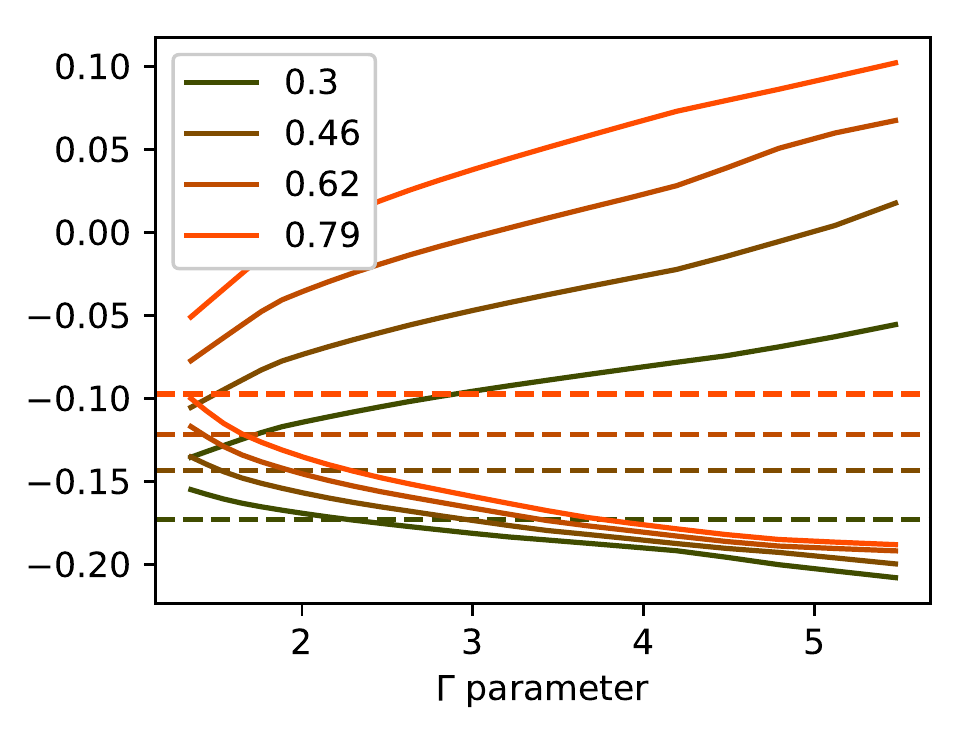}
		\caption{Varying evaluation 
			$\eta$ mixture weight from uniform  to $\pi_b^{*,u=0}$. 
		}
		\label{fig-gridworld-apx-bounds-mixture} 
	\end{minipage}\begin{minipage}[t]{0.35\textwidth}\centering
		\captionsetup{width=.8\linewidth}
		\includegraphics[width=\textwidth]{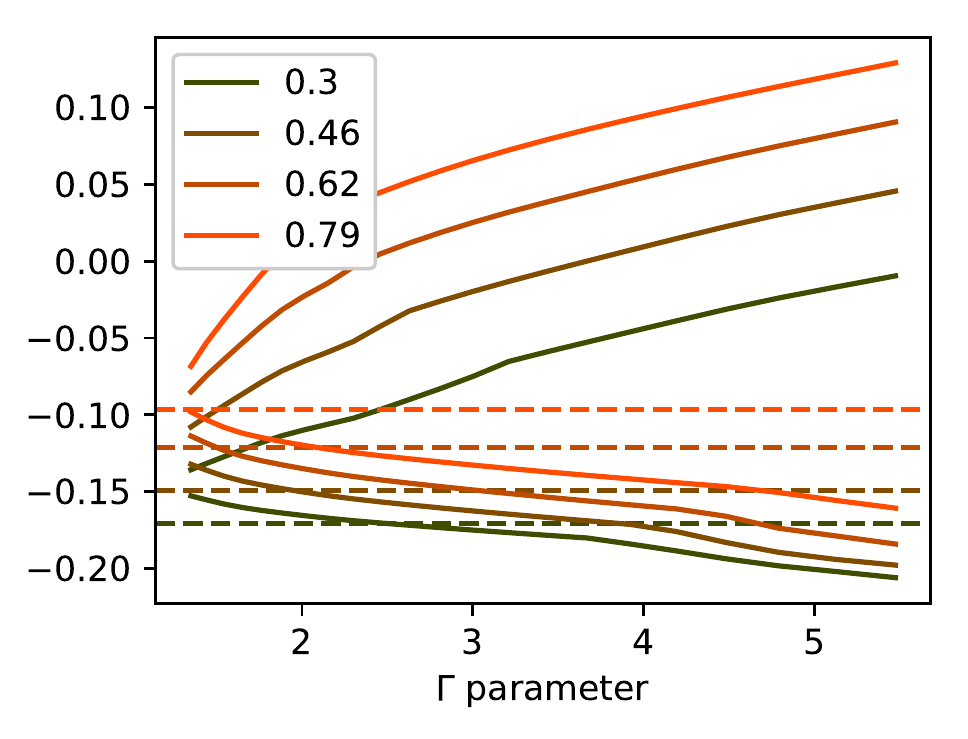}
		\caption{ Varying from uniform  to $\pi_b^{*,\Ss }$. 
		}
		\label{fig-gridworld-apx-bounds-mixturepi-s-marg} 
	\end{minipage}
\end{figure*}

\section{Empirics}\label{sec-empirics}

\paragraph{Illustrative example: Confounded random walk} We introduce a simple example in \Cref{fig-conf-random-walk}. 
\Cref{fig-conf-random-walk} satisfies a sufficient condition of \cref{lemma-suff-asn-no-time-dependent}
this is graphically denoted by arrows from $(s,u)$ tuples to $s'$ states. The confounded random walk is parametrized by the transition probabilities 
under action $a=1$:
$\textstyle P(s_1  \mid s_1,u_1, a=1) = p_{u_1}, \;\; P(s_1 \mid s_1,u_2, 1) = \frac 12 -p_{u_2},$
 where $u$ is generated exogenously upon transiting a state $s$. Transitions are antisymmetric under action $a=2$, 
 $
 P(s_1\mid  s_1,u_1, 1) = \frac 12 - p_{u_1}, \;\; P(s_2\mid s_1,u_2, 1) = p_{u_2}
 $. Then, a stationary policy that is uniform over actions generates a random walk on $\Ss$. See \cref{apx-empirics} for a full description.

In \cref{fig-conf-random-walk-varying-transitions}, we vary the underlying transition model, varying $p_{u_1} = p_{u_2}$ on a grid $[0.1,0.45]$, and we plot the varying bounds with action-marginal control variates. The true underlying behavior policy takes action $a=1$ with probability $\pi(a=1\mid s_1,u_1) = \pi(a=1\mid s_2, u_1)= \frac{1}{4}$ (and the complementary probability when $u=u_2$), modeling the setting where a full-information behavior policy is correlated with the full-information transitions. The behavior policy on $\Ss$ appears to be uniform; using these confounded estimates results in biased estimates of the transition probabilities. As we vary the transition model in \cref{fig-conf-random-walk-varying-transitions}, note that the true policy value when $\Phi = [1,2]$ is $1.5$, indicated by $\star$, and it is within the bounds for large enough $\Gamma=3$, uniformly over different data-generating processes. 
\vspace{-5pt}
\paragraph{3x3 confounded windy gridworld. }
We introduce unobserved confounding to a 3x3 simple windy gridworld environment, depicted in \Cref{fig-gridworld} in the appendix \citep{sutton1998introduction}. The the agent receives $\Phi(s) =1$ reward at a goal state but $\Phi(s) = -0.3$ at hazard states (shaded red). We assume a binary unobserved confounder $u \in \{0,1\}$ that represents ``wind strength''. 
Transitions in the action direction succeed in that direction with probability $p=0.8$, otherwise with probability $0.1$ the agent goes east or west.  However, when $u=1$, the ``westward wind'' is strong, but if the agent takes action ``east'', the agent instead stays in place (otherwise the agent transitions west). The wind is generated exogenously from all else in the environment. An optimal full-information behavior policy (agent with wind sensors) varies depending on $u_t$ by taking the left action, to avoid the penalty states. This models the setting where unobserved confounding arises due to exogenous private information. 

In \Cref{fig-stat-consistency}, we study the finite-sample properties of the bounds estimator, plotting $\textstyle\hat{\overline{R^T}}_e- \hat{\overline{R^{10k}_e}}$ for differing trajectory lengths on a logarithmic grid, $T\in[250,10000]$, and standard errors averaged over 50 replications. (We plot the difference in order to normalize by an estimate of the large-sample limit). The bounds converge; optimizing over larger $\Gamma$ values tends to increase finite sample bias. 



 We illustrate bounds obtained by our approach in 
 \Cref{fig-gridworld-apx-bounds-mixture} for evaluation policies which are mixtures to $\pi_b^{*,u=0}$, a
suboptimal policy that is optimal for the transitions when $u=0$ (no wind), and in \Cref{fig-gridworld-apx-bounds-mixturepi-s-marg},
$\pi^{*\Ss}$, a policy that is optimal on $\Ss$ given the true marginalized transition probabilities (which are unknown to the analyst, but known in this environment). 
We display the bounds as we range mixture weights on the non-uniform policy from $0.3$ to $0.8$. We display in a dashed line, with the same color for corresponding mixture weight $\eta$, the true value of $R_e$.


\section{Related work, discussion, and conclusions}\label{sec-related-work} 

\paragraph{Off-policy evaluation in RL.} 
We  build most directly on a line of recent work in off-policy policy evaluation which targets estimation of the stationary distribution density ratio \citep{hallak2017consistent,liu2018breaking,gelada2019off,kallus2019efficiently}, which can be highly advantageous for variance reduction compared to stepwise importance sampling.
\paragraph{Sensitivity analysis in the batch setting in causal inference.} 
Sensitivity analysis is a rich area in causal inference. 
A related work in 
approach is \citep{kallus2018confounding}, which builds on \citep{aronowlee12} and considers robust off-policy evaluation and learning in the \textit{one-decision-point} setting. {In \Cref{sec-apx-fh}, we discuss an extension of their or other inverse-weight robust approaches to this setting and the inherent challenges of such an approach with long horizons in introducing ``exponential robustness''.}
The identification approach in this work is very different:
the partial identification region is only identified \textit{implicitly} as the solution region of an estimating equation. 
Unlike \citep{kallus2018confounding}, \citep{bennett2019policy} consider an \emph{identifiable} setting where we are given a well-specified latent variable model and propose a minimax balancing solution. Finally, \citep{kallus2018interval,yadlowsky2018bounds} study bounds for conditional average treatment effects in the \textit{one-decision-point} setting. \cite{yang2018sensitivity} consider sensitivity analysis under assumptions on the outcome functions for a marginal structural model.

\paragraph{Off-policy evaluation in RL with unobservables.}
Various recent work considers unobserved confounders in RL. \citep{oberst2019counterfactual} considers identification of counterfactuals of trajectories in an POMDP and SCM model. 
\citep{tennenholtz2019off} study off-policy evaluation in the POMDP setting, proposing a ``decoupled POMDP'' and leveraging the identification result of \citep{miao2018identifying}, viewing previous and future states as negative controls. \citep{lu2018deconfounding} propose a ``deconfounded RL method'' that builds on the deep latent variable approach of \citep{louizos2017causal}
. 
\citep{zhang2019near} uses partial identification bounds to narrow confidence regions on the transition matrix to warm start the UCRL algorithm of \citep{jaksch2010near}.

These generally consider a setting with sufficient assumptions or data to render policy values \emph{identifiable}, where in the general observational setting they are unidentifiable.
Specifically, \citep{tennenholtz2019off} require an invertibility assumption that implies in a sense that we have a proxy observation for \emph{every} unobserved confounder, \citep{lu2018deconfounding} assume a well-specified latent variable model, also requiring that every unobserved confounder is reflected in the proxies, and \citep{zhang2019near} consider an online setting where additional experimentation can eventually identify policy value.
Our approach is complementary to these works: 
we focus on the time-homogeneous, infinite-horizon case, and are agnostic to distributional or support assumptions on $u$.
Our structural assumption on $u$'s dynamics (\Cref{asn-no-time-dependent-confounding}) is also new.

In contrast to POMDPs in general, 
which \textit{emphasize} the hidden underlying state and its statistical recovery, 
our model is distinct in that we focus on rewards as functions of the observed state. 
In contrast to robust MDPs, one 
may derive 
a corresponding ambiguity set on transition matrices (see \Cref{lemma-non-rectangularity} in the appendix), but it is generally
\textit{non-rectangular} because 
the ambiguity set does not decompose as a product set over states, which leads to a NP-hard problem in the general case \cite{wiesemann2013robust}. See \cref{apx-pomdp} of the appendix for a fuller discussion.

\vspace{-5pt}
\paragraph{Conclusions and future work}
Our work establishes
partial identification results for policy evaluation in infinite-horizon RL under unobserved confounding. 
We showed that the set of all policy values that agree with both the data and a marginal sensitivity model can be expressed in terms of whether the value of a linear program is non-positive and developed algorithms that assess the minimal- and maximal-possible values of a policy. 
Further algorithmic improvements are necessary in order to extend our results to infinite state spaces. Finally, an important next step is to translate the partial identification bounds to robust policy learning.

\clearpage
\section*{Broader Impact} 

As a contribution to offline RL,
our work is of particular importance for RL in the context of social and medical sciences, where experimentation is limited and observational data must be used. The validity of the no-unobserved-confounders assumption is of particular importance to applied research in these fields since the presence of unobserved confounding can bias standard evaluations that assume the issue away and, unchecked, this may potentially hide harms done by the policy being evaluated or learned. Our work is the first step in developing offline RL algorithms that directly address this real, practical issue, and its primary purpose is to directly deal with such biases in the data.

That said, it is well understood that there is generally a gap between theory and practice in RL as it is applied to very complex and large-scale systems. It is therefore important to keep in mind practical heuristics, stopgaps, and approximations from applied RL when translating this work into practice. The systematic investigation of the use of these in the context of confounded offline RL and more extensive experimental studies in larger environments may require additional future work. 

Moreover, there are several general potential dangers to be cognizant of when applying any offline RL tool in practice. First, if the observational data is not representative of the population, that is, there is a covariate shift in the state distribution, then the evaluation will reflect these biases and correspondingly be unrepresentative, under-emphasizing value to some parts of the population and over-emphasizing value to others. More generally, even without covariate shift, here we focused on evaluation of \emph{average} welfare, which may average the harms to some and the benefits to others; it may therefore be important in some applications to also conduct auxiliary evaluations on certain protected subgroups to ensure equal impact, which can be done by segmenting the data. Third, it is possible that any offline RL approach may be applied inappropriately when the assumptions are not met -- here we dealt directly with how to deal with violations of unconfoundedness assumptions but there may still be other assumptions such as our \cref{asn-recurrent} -- and this will mislead any evaluation or learning. For example, concerns about violations of absolute continuity of importance weights (overlap) are especially relevant in the offline RL setting. Therefore, we strongly recommend considering such offline RL and, in particular the sensitivity analyses we developed herein, as a way to inform further investigation, additional data collection, and/or investment in a randomized trial, rather than as an outright replacement for any of these. That said, offline evaluation, especially robust evaluation as we propose herein, is crucial for assessing policies \emph{before} even trialling them, where they may effect actual negative impacts on study populations.


	\bibliography{sensitivity_and_rl}
\bibliographystyle{icml2020}

\clearpage
\onecolumn

\appendix

\paragraph{Outline of the appendix} 

\begin{itemize}
	\item \Cref{apx-discussion} contains omitted discussion and comparison to other frameworks.
	\item \Cref{apx-partialid} contains proofs regarding the model and partial identification thereof. 
	\item \Cref{apx-optalg} discusses optimization and algorithms. 
	\item \Cref{apx-consistency} proves statistical consistency.
	\item \Cref{apx-empirics} contains additional empirics and computational discussion.
\end{itemize}

\section{Additional context}\label{apx-discussion}

\subsection{Off Policy Policy Evaluation: Relationship to finite-horizon case}\label{sec-apx-fh}
To aid comparison to the off-policy policy evaluation literature, we describe an approach to robustness under unobserved confounding which might be pursued in the finite-horizon case, which does not leverage stationarity.
Such an approach would bound the density ratio product of true behavior policy weights
$ \prod_{t \in[H]}\frac{ \pi_e(a_t \mid s_t)   }{ \pi_b(a_t \mid s_t, u_t)  } $ relative to the product of nominal inverse propensity weights, $\prod_{t \in[H]}\frac{ \pi_e(a_t \mid s_t)   }{ \pi_b(a_t \mid s_t)  }$ (including the moment restrictions according to it being a valid density ratio). It is apparent by factorizing the 
joint distribution that the true density ratio product 
would identify the policy value.
\begin{align*} 
\polval_e=  & \E[ \sum_{h=1}^H  r_h \mid a_{1:H } \sim \pi_e ]\\
=&\E_{b} \left[ 
\left( \frac{p^{(0)}(s_0, u_0) }{p^{(0)}(s_0, u_0) } \prod_{t \in[H]}\frac{ \pi_e(a_t \mid s_t) p(r_t\mid s_t, a_t) p(s_{t+1}, u_{t+1} \mid s_t, u_t, a_t) }{ \pi_b(a_t \mid s_t, u_t) p(r_t\mid \mathcal{H}_t) p(s_{t+1}, u_{t+1}\mid s_t, u_t, a_t)} \right) \sum_{h=1}^H  r_h  \right]\\
=&\E_{b} \left[ 
\prod_{t \in[H]}\frac{ \pi_e(a_t \mid s_t)   }{ \pi_b(a_t \mid s_t, u_t)  }  \sum_{h=1}^H r_h  \right]
\end{align*} 

While optimizing bounds on the range of the product of importance-sampling weights, $ \prod_{t \in[H]}\frac{ \pi_e(a_t \mid s_t)   }{ \pi_b(a_t \mid s_t, u_t)  } $
may be tractable 
using geometric programming \cite{boyd2007tutorial}, enforcing the moment restrictions on density ratios (as in \cref{eqn-conditional-moment}) introduces further difficulty (nonconvex equality constraints). The core difficulty is that considering an uncertainty set which decomposes as a product set over timesteps 
may be too conservative to be useful in practice. 

Further, in the infinite horizon setting, the divergence of the robust density ratio products grows with the horizon length. This is immediate when considering a relaxation of the sharp uncertainty set, i.e. optimizing the density ratio within bounds without enforcing moment constraints that must be satisfied by inverse probabilities by a geometric programming reformulation \cite{boyd2007tutorial}; verifying similar properties in the sharp setting with additional nonlinear constraints may require additional analysis. 

\subsection{Related work and relation to POMDPs and robust MDPs}\label{apx-pomdp}
\paragraph{Contrast to POMDPs.} 
In contrast to POMDPs in general, 
which \textit{emphasize} the hidden underlying state, 
our model is distinct in that we focus on rewards as functions of the observed state. The unobserved confounder is therefore a ``\textit{nuisance}'' confounder which prevents us from estimating policy value, rather than the true underlying state to recover. In settings where unobserved confounders are of concern in observational data in causal inference, typically it is unclear whether or not a latent variable model such as those underlying POMDPs indeed generalizes to the time of deployment: assuming so corresponds to a structural assumption about the environment.  

\paragraph{Contrast to Robust MDPs.} Robust MDPs, representing a \textit{model-based} approach, consider policy evaluation or improvement over an ambiguity set of the \textit{transition probabilities} \cite{i-robustmdp-05,wiesemann2013robust,nilim2005robust}. 
Alternatively, some approaches build confidence regions from concentration inequalities \cite{thomas2015high,pgc16} and restrict recommendations within them.
\citep{petrik2014raam} improve performance guarantee bounds for state aggregation in MDPs; but in their setting they are able to sample additional full-information transitions
unlike our fully-observational data setting
The difficulty in applying the robust MDP framework using an ambiguity set on transition matrices 
suggested from \Cref{lemma-non-rectangularity} (in the appendix) is
\textit{non-rectangularity} because 
the ambiguity set does not decompose as a product set over states, which leads to a NP-hard problem in the general case \cite{wiesemann2013robust}.

	\section{Proofs: Model and Partial Identification}\label{apx-partialid}

\begin{proof}[Proof of \Cref{lemma-suff-asn-no-time-dependent}]
	This is apparent from a recursive definition of $w(j,u)$: 
	\begin{align*} w(s,u) &= \lim_{T\to \infty} \frac{ \sum_{t=1}^T  \sum_{a,k,u'} p(j,u \mid k,u', a) p^{t-1}_e (k,u', a)   }{ \sum_{n=1}^T  \sum_{a,k,u'} p(j,u \mid k,u', a)  p^{t-1}_b (k,u', a) } \\
	&= \lim_{T\to \infty} \frac{ \sum_{t=1}^T  \sum_{a,k,u'} p(j,u'' \mid k,u', a)  p^{t-1}_e (k,u', a)   }{ \sum_{n=1}^T  \sum_{a,k,u'} p(j,u'' \mid k,u', a)  p^{t-1}_b (k,u',a) } = w(s,u'')
	\end{align*} 
\end{proof}


\begin{proof}[Proof of \Cref{prop-lumpable-infhorizon}]
Clearly, by assumption of Markovian dynamics on the full information state space, $w(s,u) 
	$ solves the estimating equation (state-action flow equations) on the space of $\Ss \times \mathcal U $, 
		\begin{align}
	&\E
	[ \pi_e(a \mid s) \wstat (s,u) \Wpi(a \mid s,u)  \mid s' = k ]= \wstat(k,v)~~\forall k \in \Ss, v \in \Uu \\
	&\E[w(s,u)] = 1.
	\end{align}

	We will proceed to show that under \Cref{asn-no-time-dependent-confounding}, the projected $\tilde{w}(k)$ defined as $\tilde{w}(k) \defeq w(k,v)$ equivalently solves the estimating equation on the observed state space $\Ss$.

	The forward implication of \Cref{prop-lumpable-infhorizon} follows from Theorem 1 of \citet{liu2018breaking} applied to the state variable $(s,u)$ after recognizing that $w(s,u)$ only depends on $s$ under \Cref{asn-no-time-dependent-confounding} and marginalizing out $u'$. 
	\begin{align*}
	\wmarg(k) =\frac{\Pstat_e (k)}{\Pstat_b (k)} =   \wstat(k,v) &  = \E_{(s,u), a, (s',u') \sim \Pstat_b} [ \wstat(s,u) \pi_e(a\mid s) \Wpi(a\mid s,u)   \mid s' =k, u'= v ] \\ 
	&= \E_{(s,u), a, s' \sim \Pstat_b} [ \wstat(s,u) \pi_e(a\mid s) \Wpi(a\mid s,u)   \mid s'=k] && \text{by \Cref{asn-no-time-dependent-confounding}}\\
	& = \E_{(s,u), a, s' \sim \Pstat_b} [ \wmarg(s) \pi_e(a\mid s) \Wpi(a\mid s,u)   \mid s'=k]\\
	\end{align*}

Verifying that $\tilde\wstat(k) p^{(\infty)}_{b}(k)$ satisfies conditions on the invariant measure on $s$: 
	\begin{align*}
	\tilde\wstat(k) 
	&= \frac{1}{p^{(\infty)}_{b}(k)} \sum_{j,a} \tilde\wstat(j)\sum_u p(k \mid j,u,a)  p^{(\infty)}_{b}(s,u) \pi_{b}(a\mid s,u)
	\frac{\pi(a\mid s)}{\pi_{b}(a\mid s,u)} , \forall k\\
	&= 	\frac{1}{p^{(\infty)}_{b}(k)}
	\sum_{j,a} \tilde\wstat(j)\pi(a\mid j)\sum_u p^{(\infty)}_{b}(j,u) p(k \mid j,u,a)   , \forall k
	\end{align*}
	
Therefore,
	$$ 	\tilde\wstat(k) p^{(\infty)}_{b}(k) = \sum_{j,a} \tilde\wstat(j)p^{(\infty)}_{b}(j) \pi(a\mid j) p (k \mid j, a) , $$
	
so we conclude that $\tilde\wstat(s) \propto p^{(\infty)}_{e}$, the stationary distribution on $\Ss$ induced by $\pi_e$.
	
	Finally, we argue the reverse implication; uniqueness of the solution of $\tilde w(s)$. Uniqueness is a consequence of the positive recurrence assumption (\Cref{asn-recurrent}) on the full-information MDP on $\Ss \times \uset$.	Note that by definition of recurrence, recurrence on the full-observation state space of the Markov process induced under $\pi$ implies recurrence of the Markov process induced under $\pi$ on its marginalized transitions $p(k \mid j,a)$. Recurrence requires that starting from any state $j,u$ in the recurrent class, the number of visits of the chain to the state is infinite. Clearly, if this is satisfied by the full-information transition matrix, this is also satisfied for the aggregated recurrent class corresponding to marginalized transitions. 
	
	Therefore, the stationary distribution exists, and is unique on $\Ss$, under the marginalized transition matrix induced by $\pi_e$. The solution to the invariant measure flow equations on $ \Ss$ satisfies that $\tilde w(s) \Pstat_b(s) \propto \Pstat_e(s)$; and only $\tilde w(s)$ satisfies this requirement. 
\end{proof}

\subsection{Observable Implications, Membership Oracle, and Sharpness}\label{apx-obs-implications}
In this section, we introduce the state-action polytope and the state-action-state frequency polytope and deduce the observable implications which lead to \Cref{prop-sensitivity-sharpness}, the main membership certificate result of \Cref{lemma-weight-reparametrization}, and \Cref{thm-sharpness}.

For any infinite-horizon MDP with transition matrix on $\Ss \times \Uu$, the stationary dynamics 
impose restrictions on
the unknown full-information state-action-state visitation distribution, $\Pstat_b(s,u,a,s',u')$, and its observable marginalization $\Pstat_b(s,a,s')$. 
These restrictions
are encapsulated as the full-information \textit{state-action polytope} $(\op{SAP})$, which is the set of all limiting state-action occupancy probabilities achievable under {any} policy, and 
the
closely related \textit{state-action-state} polytope $(\op{SASP})$ \citep{puterman2014markov,mannor2005empirical}.
%
Marginalizing the full-information constraints with respect to $\Pstat_b(u \mid s)$ leads to the marginalized versions $\op{mSAP}$ and $ \op{SASP}$.

\begin{proof}[Proof of \Cref{prop-sensitivity-sharpness}]
	To study the observational implications of $\op{SAP}, \op{SASP}$ (e.g. the implications of the full-information polytopes which are additionally enforceable as constraints on $\tilde \Wset$), we study the marginalized versions of both the state-action polytope and the state-action-state frequency polytope under the behavior policy as studied in \cite{puterman2014markov,mannor2005empirical}. Typically the extremal analysis of the state-action polytope in the infinite-horizon case  characterizes the structure of the optimal policy. We simply focus on its properties as a characterization of all the possible limiting state-action frequencies under \textit{any} stationary policy. 
	
	While \cite{mannor2005empirical} studies the asymptotic inclusion of
	more general policies such as non-stationary policies, we focus on the case of stationary policies for simplicity. While we state the following analysis for discrete state and action spaces (e.g. discrete $u$), the discussion of \citet[Sec. 4]{altman1991markov} provides regularity results on the set of limiting state-action measures for the continuous state case with continuous $u$. One sufficient condition for the case of continuous $u$ is that, indexing $x,y$ into abstract state tuples on $\Ss \times \Uu$, for the transition probability density defined as $P^\pi_{x,K} \defeq \sum_{y \in K} P^\pi_{x,y}$, with $P^\pi_{x,y} = \mathbb P(s_{t+1} = y\mid s_t =x)$, given any $\epsilon > 0$ there exist a finite set $K(\epsilon)$ and an integer $N(\epsilon)$ such that for all $x \in X$ and $g \in U(S)$, $[(P^\pi)^{N(\epsilon)}]_{x,K(\epsilon)} \geq 1 - \epsilon$. 
	
	%
	First, we introduce $\op{SAP}, \op{SASP}$ as studied in \cite{mannor2005empirical}.
	\begin{definition}[State-action polytope]
		Given an MDP, the \textbf{state-action polytope} $\op{SAP}$ is defined as the set of vectors $x$ in $\Delta^{\mathcal S \times \mathcal A}$ that satisfy 
		\begin{equation}\label{eqn-X-flow} 
		\sum_{j,u'} \sum_a p(k,u''\mid j,u',a) \Pstat_b(j,u',a) = \sum_{a'} \Pstat_b(k,u'', a'), \qquad \forall s' 
		\end{equation}
		
	\end{definition}
	
	where $ \Pstat_b(j,u',a)\in \Delta^{\mathcal S \times \mathcal A}$ is the limiting expected state-action frequency vector under policy $\pi$. This constraint can be understood as a ``flow conservation'' constraint satisfied by any full-information joint distribution $p_b^\infty(s,u,a,s',u')$. 

	\begin{definition}[State-action-state polytope]
		The \textbf{state-action-state} frequency polytope, $\op{SASP}$, is the set of vectors in $\Delta^{\mathcal S \times \mathcal A \times \mathcal S}$ which satisfy conformability of state transitions under transition probabilities; and marginalization: 
		
		\begin{equation}\label{eqn-Q-conformability-transitions}
		\Pstat_b(j,u',a,k,u'') = p(k,u''\mid j,u',a)  \sum_{\tilde k,\tilde u''} \Pstat_b(j,u',a,\tilde k,\tilde u''), \qquad \forall j,u', a ,( k, u'')
		\end{equation} 
		\begin{equation}\label{eqn-Q-flow-conservation}
		\sum_{j,u'}  \sum_a \Pstat_b(j,u',a,k,u'') = \sum_{a'} \sum_{\tilde k,\tilde u''} \Pstat_b(k,u'',a',\tilde k,\tilde u''),\qquad \forall (k,u'')
		\end{equation}
	\end{definition}
	Lemma 3.1 of \cite{mannor2005empirical} states that the two sets are equivalent: if $\Pstat_b(s,u,a)  \in \op{SAP}$, and furthermore under the transformation $$\Pstat_b(j,u',a,k,u'') =\Pstat_b (j,u',a) p(k,u''\mid j,u',a),$$ then $\Pstat_b(j,u',a,k,u'')\in \op{SASP}$. Every element of $\op{SASP}$ can be generated in this manner from some element of $\op{SAP}$. It is this construction that leads to the ``information loss'' stated formally in the next result which characterizes the marginalized versions of \cref{eqn-Q-conformability-transitions,eqn-Q-flow-conservation,eqn-X-flow} in terms of the marginalized weight of the reparametrization, $g_k(a\mid j)$.

	\begin{lemma}\label{lemma-sharpness-Q}
		The marginalized version of \cref{eqn-X-flow} is 
		\begin{equation}\label{eqn-sharp-X-compatibility}  \Pstat_b(k) = \sum_a \sum_j \Pstat_b(j) p(k\mid j,a ) g_k(a\mid j)^{-1}, \; \forall k \in \Ss 
		\end{equation} 
		
		The marginalized version of \cref{eqn-Q-conformability-transitions} enforces conformability of $g_k(a\mid j)$ for the true non-identifiable transitions,
		\begin{equation}\label{eqn-sharp-Q-compatibility}
		p(k \mid j,a )  = \Pstat_b(a,k\mid j)g_k(a \mid j), \;  \forall j,a, k
		\end{equation} 
		and of \cref{eqn-Q-flow-conservation} is the conditional control variate, 
		\begin{equation}\label{eqn-sharp-Q-flow}\Pstat_b (k \mid a) = \sum_{j} \Pstat_b (j,  a,k) g_k(a\mid j), \; \forall k, a 
		\end{equation}
	\end{lemma} 
	
\begin{proof}[Proof of \Cref{lemma-sharpness-Q}]

	\textbf{Marginalizing \cref{eqn-Q-conformability-transitions}:}
	
	Starting from the compatibility restriction with the observed empirical state-action frequencies: $\sum_u \Pstat_b (j,u) \pi_b(a \mid j,u)  p(k,u' \mid j,u,a )= \Pstat_b(j,a,k,u')$
	and marginalizing over $u'$: 
	\begin{align*}
	&\sum_{u} \Pstat_b (j,u) \pi_b(a \mid j,u)  p(k\mid j,u,a )=\sum_{u}  \Pstat_b(j,a,k) \\
	&p(k \mid j,a ) \Pstat_b (j)  \sum_{u} \Pstat_b (u\mid j ) \pi_b(a \mid j,u)  \frac{p(k\mid j,u,a )}{p(k \mid j,a )}= \Pstat_b(j,a,k)\\
	& p(k \mid j,a )  g_k(a \mid j)^{-1}= \Pstat_b(a,k\mid j)
	\end{align*}
	
	where $p(k \mid j,a ) = \sum_u p(k\mid j,u,a )\Pstat_b (u\mid j )$ and $ \Pstat_b(a,k\mid j) = \frac{ \Pstat_b(j,a,k)}{ \Pstat_b(j)} $. 
	
	This leads to the conformability requirement of $g_k(a\mid j)$ for the marginal transition probabilities with respect to the observational joint distribution.
	\begin{equation}p(k \mid j,a )  = \Pstat_b(a,k\mid j)g_k(a \mid j)
	\end{equation} 
	which can also be derived as a marginalization of \cref{eqn-Q-conformability-transitions}: 
	$$\sum_{u,u'}	\Pstat_b(j,u,a,k,u') = \sum_{u,u'}p(k,u'\mid j,u,a)  \sum_{\tilde k,\tilde u'} \Pstat_b(j,u,a,\tilde k,\tilde u'), \qquad \forall j,k$$
	
	\textbf{Marginalizing \cref{eqn-Q-flow-conservation}}:

	Recall that from the forward decomposition of joint distribution with respect to the transition from $j,u \to k,u'$, but conditioning on current state and action, we have that:
	$$ \Pstat_b(k,u'\mid a, j,u) = \frac{\Pstat_b( j,u, a ,k,u')}{\pi_a(j,u) \Pstat_b(j,u) } $$
	Then marginalize the definition of $\Pstat_b (k ,u'\mid a)$ over $u'$ to obtain $p_b^\infty(k\mid a )$:
	\begin{align*}
	\sum_{u'} \Pstat_b (k ,u'\mid a) &= \sum_{u'} \sum_{j,u} \Pstat_b (k,u' \mid a,j,u ) \Pstat_b(j,u) = \sum_{u'} \sum_{j,u} 
	\frac{\Pstat_b (j,u ,  a,k,u') \Pstat_b(j,u) }{\pi(a\mid j,u) \Pstat_b(j,u) }\\
	& =\sum_{j} \Pstat_b (j,  a,k)
	\sum_u 
	\frac{\Pstat_b (j,u ,  a \mid k) }{ \Pstat_b (j,  a \mid k) \pi(a\mid j,u)  } \\
	&= \sum_{j} \Pstat_b (j,  a,k) g_k(a\mid j)
	\end{align*}
	
	Therefore: 
	$$ \Pstat_b (k \mid a) = \sum_{j} \Pstat_b (j,  a,k) g_k(a\mid j), \forall k, a $$
	
\end{proof}
	
	In order to interpret \textit{which} of these constraints are observable implications and which are ultimately informative, we next leverage a structural characterization that $g_k(a \mid j)$ can be interpreted as the function which renders the transition probabilities conformable to the joint distribution. Its proof is of independent interest in establishing the relationship to robust MDPs. We introduce the \textit{biased} marginalized transition probabilities $\tilde{p}(s' \mid s,a)$, which would be obtained from naive estimation from the observational joint distribution:
	$$ \tilde{p}(s' \mid s,a) \defeq \frac{ \Pstat_b(s,a,s')}{\Pstat_b(s) \; \E[ \pi_b(a \mid s,u)\mid s ]  } =  \Pstat_b(s)^{-1 }\sum_u \frac{\Pstat_b(s,u,a,s')}{ \E[ \pi_b(a \mid s,u)\mid s ] }
	$$
	These are biased estimates because they do not appropriately account for the transitions under the true $\pi_b(s,u)$ policy, only its marginalization over $s$. 
	\begin{lemma}\label{lemma-non-rectangularity}
		$$
		{\tilde{p}(k \mid j,a )} \pi_b(a\mid j){} = { p(k \mid j,a )g_k(a\mid j)}{}$$
	\end{lemma}
	\begin{proof}[Proof of \Cref{lemma-non-rectangularity}]
	
	With full information, the transition probabilities could be estimated as 
	$$
	p(s', u'\mid s, u,a) =
	\frac{ \Pstat_b(s,u,a,s',u')}{\pi_b(a \mid s,u) \Pstat_b(s,u)}  $$
	and similarly, the marginalized transition probabilities as $\frac{ \Pstat_b(s,u,a,s')}{\pi_b(a \mid s,u) \Pstat_b(s,u)} = p(s' \mid s, u,a)$.
	
	A model-based perspective would partially identify the transition matrix under $\pi_e$, deduce the bounds of $p(s' \mid s,a)$ relative to $\tilde{p}(s' \mid s,a)$. Note that under \Cref{asn-no-time-dependent-confounding}, 
	\begin{equation}\label{eqn-true-psas}
	p(s' \mid s,a) = \sum_u p(s' \mid s,u, a) \Pstat_b(u)
	= \sum_u \frac{\Pstat_b(s,u,a,s')\Pstat_b(u)}{\pi_b(a \mid s,u) \Pstat_b(s,u) }
	= \Pstat_b(s)^{-1 }\sum_u \frac{\Pstat_b(s,u,a,s')}{\pi_b(a \mid s,u)  }
	\end{equation} 
	and further the distribution on unobserved confounders is independent of the policy, $\Pstat_b(u) = \Pstat_e(u)$. 
	\begin{equation}\label{eqn-nominal-psas}\tilde{p}(s' \mid s,a) = \frac{ \Pstat_b(s,a,s')}{\Pstat_b(s) \; \E[ \pi_b(a \mid s,u)\mid s ]  } =  \Pstat_b(s)^{-1 }\sum_u \frac{\Pstat_b(s,u,a,s')}{ \E[ \pi_b(a \mid s,u)\mid s ] }
	\end{equation}
	
	Combining \Cref{eqn-true-psas,eqn-nominal-psas} yields the statement of the lemma. 
	
\end{proof}
	\Cref{lemma-non-rectangularity} shows that 
	the constraints in \Cref{eqn-sharp-Q-compatibility,eqn-sharp-X-compatibility} are uninformative:
	further restricting $\tilde{p}(s' \mid a,s)$ within the given range of $p(s'\mid a,s)$ is redundant. Another interpretation is that $g_k(a\mid j)$ are precisely the weights which render the observed stationary occupancy distribution $\Pstat_b(s,a,s')$ conformable under the \textit{unobserved} true marginal transition probabilities. 
	
	\cref{prop-sensitivity-sharpness} follows from \cref{lemma-sharpness-Q,lemma-non-rectangularity}.
	
\end{proof}

\begin{proof}[Proof of \Cref{lemma-weight-reparametrization}]

We verify that optimizing over $F(w) = 0 \iff w \in \Theta, \E[w] = 1$ by first showing the reparametrization of $F(w)$ with respect to $g$, and then leveraging the characterization of \Cref{prop-sensitivity-sharpness} and sharpness argument to verify that $F(w) = 0 \iff w \in \Theta, \E[w] = 1$.

	\textbf{Step 1: Proving the reparametrization of $F(w)$ with respect to $g_k(a \mid j)$, and reformulating $g_k(a \mid j)$. }

We first expand the sample expectations for the estimating equation of \Cref{prop-lumpable-infhorizon}: 
	
	$$ \sum_j  \sum_{i=1}^N \sum_{t=0}^T\sum_{a,u}  
	\frac{\mathbb{I}[(s_t^{(i)}=j, u_t^{(i)} = u),  a_t^{(i)} =a, s_{t+1}^{(i)} = k] }{ p( s_{t+1}^{(i)} = k) }
	\left( \wstat(k) - \wstat(j) 
	\Wpi^{(i)} (a\mid j,u)
	\pi_e (a\mid j) \right) = 0, \forall k
	$$

Rewrite as an expectation with respect to the observational (identifiable) joint distribution $\Pstat_b ( j, a  \mid k  )$, taking limits as $T\to \infty, N\to \infty$ and multiplying by $\frac{\Pstat_b ( j, a  \mid k  )	}{\Pstat_b ( j, a  \mid k  )	}
	=1$: 
	\begin{align*}
	&w(k) - 
	\sum_j \wstat(j)  \sum_{a} \pi_e (a\mid j) 
	\Pstat_b ( j, a  \mid k  )		\underbrace{\sum_{u}  
		\frac{\Pstat_b ( j,u, a  \mid k  )}{\Pstat_b ( j, a  \mid k  )}
		\Wpi (a\mid j,u)}_{\Wpimarg_k(j,a; W) }
	= 0, \forall k 
	\end{align*} 
	Therefore, dependence on $\Wpi$ arises only through the marginalized weight $g_k(a\mid j)$:
	$$\Wpimarg_k(a \mid j) =\sum_{u}  
	\frac{	p^{(\infty)}_{b} ( j,u, a  \mid k  )}{\Pstat_b ( j, a  \mid k  )	}
	\Wpi (a\mid j,u) $$
However, the marginalized weight depends on the unknown data-generating joint distribution $p^{(\infty)}_{b} ( j,u, a  \mid k  )$; and so it is unclear how to optimize over it. Next, we show that \textit{optimizing} over $g_k(a\mid j)$, which are the unknown inverse weights $\beta(a \mid j,u)$ convolved with an unknown density, is almost \textit{equivalent} to optimizing over the set of weights $\Wset$, up to a moment constraint on ensuring that the implied full-information transition probabilities are valid probability distributions, e.g. that $  \sum_k p(k \mid j,u,a) = 1$. We first we show that it is equivalent to optimize over $g_k(a \mid j)$ over the same bounds, though this may not enforce the restriction $ \sum_k p(k \mid j,u,a) = 1$. In the next step, we will argue that this restriction to valid transition probabilities is enforced by the feasibility of $g$ for the estimating equation. 

Define 
$$\tilde{\Wset}' \defeq \{ g\in \mathbb{R}^{\nS  \nA  \nS }_+ 
\colon
\exists \Wpi \in \Wset 
\text{ such that } 
g_k(a\mid j) = \sum_u \frac{p^{(\infty)}_{b} ( j,u, a  \mid k  )}{p^{(\infty)}_{b} ( j, a  \mid k  )}\Wpi
 \} $$
as the ambiguity region of $g_k(a\mid j)$ induced by restrictions on $\beta \in \Wset$. We show how to identify elements $\tilde \beta \in\tilde{\Wset}' $ with corresponding elements $\beta \in \Wset$. Although $p^{(\infty)}_{b} ( j,u, a  \mid k  )$ is not identifiable from observed data, its marginalization over $u$, $p^{(\infty)}_{b} ( j, a  \mid k  )$, is identifiable, so we can partially identify $\tilde{\Wset}' $ as follows: 
	\begin{align*}
	\tilde{\Wset}' =	\left\{ 
 g\in \mathbb{R}^{\nS  \nA  \nS }_+
 \colon 
	\begin{matrix}
& g_k(a \mid j) =	\sum_u \frac{p^{(\infty)}_{b} ( j,u, a  \mid k  )}{p^{(\infty)}_{b} ( j, a  \mid k  )}
\Wpi (a\mid j,u)  \\
	& \sum_{u}   p^{(\infty)}_{b} ( j,u, a  \mid k  ) = p^{(\infty)}_{b} ( j, a  \mid k  ) 
	\\&
	0 \leq p^{(\infty)}_{b} ( j,u, a  \mid k  ) \leq 1
	\\& 
	\beta \in \mathcal B
	\end{matrix}
	\right\}
	\end{align*}

	A simple reparametrization with respect to $
	q_{j,u,a \mid k} 
	\defeq \frac{p^{(\infty)}_{b} ( j,u, a  \mid k  )}{ p^{(\infty)}_{b} ( j, a  \mid k  )}$ shows that optimizing over $\tilde{ \mathcal B}'$ is equivalent to optimizing over elements of $\mathcal B$ averaged by unknown weights on the {simplex}. In the following, suppress dependence of $q_{j,u,a \mid k} \beta(a\mid j,u)$ on $a,j$ for brevity, and let $q_{j, \cdot, a, \mid k}$ denote the vector . 
	\begin{align*} 
	\tilde{\Wset}'
	= \left\{ 	q_{j, \cdot, a, \mid k}^\top \beta ~~\colon ~~
		\beta \in \Wset ;\;~~
		q^\top_{j, \cdot, a, \mid k} 1  = 1, \;
	0 \leq q_{j, \cdot, a, \mid k}   \leq 1, ~~\forall j,a,k\;
	\right\}
	\end{align*}
	In particular this suggests that $
\Wpimarg'_k(a \mid j)  \in \Wset$ since by convexity of $\Wset$, we can map $\Wpimarg'_k(a \mid j) \in \tilde{\Wset}'$ to some $\Wpi \in \Wset$; $q$ are the convex combination weights. In the other direction, clearly any $\Wpi' \in \Wset $ is realizable by a $q$ which is a Dirac measure which selects $\Wpi'$, so $\Wpi' \in \tilde{\Wset}'$.
	
	Lastly, we directly verify the control variate property (\cref{eqn-conditional-moment} corresponding to \cref{eqn-conditional-ctrlvar}) that $\sum_j 	\sum_k 	p^{(\infty)}_{b} ( j, a  \mid k  )p_b(k) {g}_k(a \mid j) = 1$:
	\begin{align*}
	\sum_j	p^{(\infty)}_{b} ( j, a  , k  ){g}_k(j,a) & =  \sum_{j,u} 	p^{(\infty)}_{b} ( j, a  \mid k  ) p_b^{\infty}(k) \frac{p^{(\infty)}_{b} ( j,u, a \mid k  )}{p^{(\infty)}_{b} ( j, a  \mid k  )} {\Wpi}(a \mid j, u) \\
	&=  \sum_{j,u}  \frac{p^{(\infty)}_{b} ( j,u) \pi_b(a\mid j,u) p( k \mid j,u,a)   }{p_b^{\infty}(k)}p_b^{\infty}(k) {\Wpi}(a \mid j, u)  \\
	&=    \sum_{j,u} {\Pstat_b( k , j,u \mid a)   }
	\end{align*}
	so that we verify the action-marginal control variate, $	\sum_k  \sum_j	p^{(\infty)}_{b} ( j, a  , k  ){g}_k(a \mid j)  = 1 $. Note that this further implies \cref{eqn-conditional-moment}:
	$$	\sum_j	p^{(\infty)}_{b} ( j, a  , k  ){g}_k(j,a) = 
	\sum_{j,u} {\Pstat_b( k , j,u \mid a)   } = \Pstat_b(k \mid a)
	$$

Finally, to help interpret $g_k(a\mid j)$, we may further simplify and observe that 
	\begin{align*}
	\Wpimarg_k(a\mid j ) &= \sum_{u}  
	\frac{p^{(\infty)}_{b} ( (j,u), a  \mid k  )}{ p^{(\infty)}_{b} ( j, a  \mid k  )}
	\Wpi (a\mid j,u)   =
	\frac{ \sum_{u}  
		\Pstat_b(u\mid j ) p(k\mid j,u, a)
	}{ 
		\sum_u \Pstat_b(u\mid j ) p(k\mid j,u, a) \pi_b (a \mid j,u) 
	}\\
	& =
	\frac{ p(k\mid j, a)
	}{ 
		\sum_u \Pstat_b(u\mid j ) p(k\mid j,u, a) \pi_b (a \mid j,u) 
	}\\
	& = \frac{ 1
	}{ 
		\sum_u\pi_b (a \mid j,u) 	\frac{ \Pstat_b(u\mid j ) p(k\mid j,u, a)}{p(k\mid j, a) } 
	}
	\end{align*}
	
	Note that in this step, we did \textit{not} require the extremal $g_k(a\mid j)$ to be achieved by valid unobserved transition probabilities such that $\sum_k p(k\mid j,u, a)=1$. To do so directly would introduce technical difficulties in effectively requiring a bilinear formulation. Instead, we have shown this one-to-one correspondence of $\tilde{\mathcal B}'$ with $\beta \in \tilde{\mathcal B}$ occurs for the marginalized weights, which are not further required to correspond to valid adversarial transition probability distribution. In the next steps, we argue that feasibility of $ g \in \tilde{\mathcal B}'$ for $F(w) \leq 0$ ensures this compatibility with valid transition matrices.
	
\vspace{10pt}
	
\textbf{Step 2: Proving $F(w) \leq  0, \E[w] = 1 \implies w \in \Theta $:}

		\Cref{prop-sensitivity-sharpness} shows that the specification of $\tilde\Wset$ exhausts the observable implications of the sharp full information polytope of all limiting state-action-state occupancy probabilities. 
	It remains to show that $w$ is feasible for some $g_k(a\mid j) \in \mathcal B$ iff $g_k(a \mid j)$ satisfies \cref{eqn-sharp-Q-compatibility}, $p(k \mid j,a )  = \Pstat_b(a,k\mid j)g_k(a \mid j), \;  \forall j,a, k$.
	
	First we show that $F(w) \leq 0, \E[w] =1 $ implies $w \in \Theta$.	Suppose $g$ is feasible for the estimating equation: 
	\begin{equation}\label{eqn-est-eqn-sharpness-pf} 
	w(k) - 
	\sum_j \wstat(j)  \sum_{a} \pi_e (a\mid j) 
	\Pstat_b ( j, a  \mid k  )		\Wpimarg_k(a \mid j)
	= 0, \forall k 
	\end{equation}
	
	
	Next we verify that $g_k(a\mid j)$ satisfying \cref{eqn-sharp-Q-compatibility} is feasible for the estimating equation for $w$, \cref{eqn-est-eqn-sharpness-pf}. Note that $g_k{a\mid j}$ corresponding to underlying transitions which do not satisfy $ \sum_k p(k \mid j,u,a)=1$ cannot satisfy \cref{eqn-sharp-Q-compatibility}. By Bayes' rule and conformability of $g_k(a\mid j)$ for $p(k\mid j,a)$,
	\begin{equation}\label{eqn-bayesconformability}
	\Pstat_b(j,a \mid k) = \frac{p(k\mid j,a) \Pstat_b (j) g_k(a \mid j)^{-1}}{\Pstat_b(k)}
	,
	\end{equation} 
	so that we can verify the estimating equation holds with respect to the true marginalized transition dynamics under $\pi_e$: 
	
	\begin{align*}w(k)\Pstat_b(k) - 
	\sum_j \wstat(j) \Pstat_b (j) \sum_{a} \pi_e (a\mid j) 
	p(k\mid j,a) 
	= 0, \forall k 
	\end{align*}
	Markovianness of the induced MDP under the \textit{true} marginal transition probabilities $ p(k\mid j,a) $ (follows since $p(k\mid j,a) $ corresponds to the $\Pstat(u)$-occupancy-weighted aggregation to $\Ss$ under \Cref{asn-no-time-dependent-confounding}), then $w(k)\Pstat_b(k)$ is proportional to the invariant measure on $\Ss$. 
	
	\vspace{10pt}
	\textbf{Step 3: Proving $w \in \Theta 
		\implies
		F(w) = 0, \E[w] = 1  $:}
	
	Next we show the other direction, that $w$ feasible for  \Cref{eqn-est-eqn-sharpness-pf} for some $g_k(a\mid j)$ implies that $\Wpimarg_k(a \mid j)$ satisfies \Cref{eqn-sharp-Q-compatibility}. This direction follows once we identify $w(s)\Pstat_b(s)$ feasible for  \Cref{eqn-est-eqn-sharpness-pf} uniquely with $\Pstat_e(s)$, which follows from \Cref{asn-recurrent}. By \Cref{eqn-bayesconformability}, feasibility implies that $w(s) \Pstat_b(s)$ satisfies compatibilty under $\pi_e$ for $g_k(a\mid j)$. Uniqueness of the density ratio implies that compatibility must also hold under $\pi_b$. 
\end{proof}

\begin{proof}[Proof of \Cref{thm-sharpness}]
\Cref{thm-sharpness} is a consequence of \Cref{prop-sensitivity-sharpness} and that computing the support function of a set (e.g. optimizing an arbitrary linear objective over $\Theta$) is equivalent to optimizing over the convex hull of $\Theta$ \cite{rockafellar1970convex}, $\op{conv}(\Theta)$. Convexity of the interval and of $\op{conv}(\Theta)$ yields sharpness. 
\end{proof}

\paragraph{Relationship between \cref{lemma-non-rectangularity} and robust MDPs}
\begin{remark}
	We can define an ambiguity set for marginal transition probabilities on $\Ss$ for $s,a$ state-action pairs, $ \mathcal{P}_{s'\mid s,a}$: 
	$$ P( \cdot \mid s,a) \in  \mathcal{P}_{s'\mid s,a} \defeq \left\{ 
	P_{sa} \in \Delta^{\nS} \colon \exists \; \Wpi \in \Wset  \text{ such that }	P_{sa}(k) = \sum_u \Wpi \cdot\; \frac{\Pstat_b(s,u,a,k)}{  \Pstat_b(s) } ,\forall k\in \Ss \right\}$$
	By an analogous construction as in \Cref{lemma-weight-reparametrization}, without additional restrictions on the variation of the unobserved joint visitation distribution $\Pstat_b(s,u,a,s')$,
	$$ \mathcal{P}_{s'\mid s,a} \defeq \left\{  P_{sa} \in \Delta^{\nS} \colon \exists \; \Wpimarg \in \tilde\Wset 
	~~ \text{ s.t. }~~	P(s'\mid s,a)=  \Wpimarg_{s'}(a \mid s) \cdot\;  \tilde{P}(s' \mid s,a) \right\}  $$

	However, due to the restrictions on $g_{k}(a\mid j)$ corresponding to valid probability distributions, as well as the restrictions that $P(s'\mid s,a)$ corresponds to valid probability distributions, the feasible implicit ambiguity set on transition sets is not merely the union of $ \mathcal{P}_{s'\mid s,a}$ for all $s,a$. Instead, the valid ambiguity set combines the restrictions induced by the bounds assumptions of $\beta$ over $P_{sa} \in [\Delta^{\nS}]^{\nS \times \nA}$, with $P_{s'\mid s,a} = p(s' \mid s,a)$, and the observable implications of \Cref{prop-sensitivity-sharpness}.
	\begin{equation}\label{eqn-rmdp-ambiguity-feas-w}
	\mathcal{P}_{\text{feas}} = \left\{ P \in [\Delta^{\nS}]^{\nS \times \nA}\colon
	\;\;\;
	\begin{aligned}	  
	&\exists \;
	g \in \Wset, w
	~~\text{ s.t. }~~&\\
	& P(k\mid j,a)  = \Wpimarg_k(a\mid j)
	\cdot  \pi(a\mid j) \tilde{p}(k\mid j,a), 
	&
	\forall j,a,k \\
	&\Pstat_b (k \mid a) = \sum_{j} \Pstat_b (j,  a,k) g_k(a\mid j),
	&
	\forall k,a\\
	&\E
	[ \pi_e(a \mid s) \wstat (s) \Wpimarg_k(a \mid s)  \mid s' = k ]= \wstat(k)~~\forall k,\\
	&\E[w(s)] = 1
	\end{aligned}
	\right\} 
	\end{equation}
	Notably, the inverse probability restrictions on $g_k(a \mid j)$ render such an ambiguity set \textit{non-rectangular} over states and actions, while the requirement of compatibility (enforced by the estimating equations) additionally introduces nonconvexity.
\end{remark}

\section{Proofs: Optimization and algorithms}\label{apx-optalg}

\begin{proof}[Proof of \Cref{prop-disjunctive}]
	
	A feasibility oracle for $\Theta$, for a given $w$, is given by checking the \textit{existence} of $g \in \mathcal W $ satisfying the moment condition. For brevity, let $m_k (w, g) $ denote the $k$th moment restriction:
\begin{align*} 
m_k (w, g) &\defeq {  \sum_{j,a}
p^{(\infty)}_{b} ( j, a  \mid k  )\wstat(j) \pi_e (a \mid j) \Wpimarg_k(a\mid j)  -w(k) 
} 
\\\primalfeasible&\defeq
\underset{g \in \tilde\Wset}{\min}\;
\sum_k 
\abs{m_j (w, g) }
\end{align*}
	
	We will generate a disjunctive reformulation of $\primalfeasible$ by appealing to a different lifting of the $\ell_1$ norm that enumerates the possible sign patterns on $\{-1,1\}^{\nS}$; the next lemma briefly verifies equivalence. 
	\begin{lemma}\label{lemma-l1}
	$$\min
\{
\sum_j	 \lambda_j m_k(w, g)
\colon  g \in \tilde \Wset
\} = 0, \forall \lambda_i \in \{-1,1\}^p \iff 
  \min\{ \sum_j \abs{m_k(w, g)} \colon g \in \tilde \Wset\}= 0   $$
	\end{lemma} 

We next briefly introduce the disjunctive programming framework.

\paragraph{Preliminaries for disjunctive programs} 
First we introduce disjunctive programs with generic notation \cite{balas1998disjunctive}. A disjunctive program optimizes over the union of polyhedra. A disjunctive program is of the form: 
$\min \left\{c x | A x \geqslant a_{0}, x \geqslant 0, x \in L\right\}$ where the logical conditions $x \in L$ can be represented by the disjunctive normal form or the conjunctive normal form, 
$$
\{A x \geqslant a_{0}, \quad x \geqslant 0; \; \underset{i \in Q_j}{\vee} \left(d^{i} x \geqslant d_{i 0}\right), \quad j \in S
\},$$
or equivalently, to make the conjunctions apparent, 
$$
\begin{bmatrix}
Ax \geq a_0 \\ x \geq 0
\end{bmatrix} \wedge [ \underset{i \in Q_0}{\vee} \left(d^{i} x \geqslant d_{i 0}\right)] \wedge \dots \wedge
[ \underset{i \in Q_{\vert S \vert}}{\vee} \left(d^{i} x \geqslant d_{i 0}\right)]
$$

The linear programming equivalent of a disjunctive normal form is given by Theorem 2.1 of \cite{balas1998disjunctive}. It generically provides the linear programming formulation of a disjunctive form, $F=\left\{x \in R^{n} | \bigvee_{h \in Q}\left(A^{h} x \geqslant a_{0}^{h}, x \geqslant 0\right)\right\}$, as $$\operatorname{clconv} F=\left\{x \in R^{n} | \begin{array}{ll}x=\sum_{h \in Q^{*}} \xi^{h}, & \\ A^{h} \xi^{h}-a_{0}^{h} \xi_{0}^{h} \geqslant 0, & h \in Q^{*} \\ \sum_{h \in Q^{*}} \xi_{0}^{h}=1, & \left(\xi^{h}, \check{\xi}_{0}^{h}\right) \geqslant 0, h \in Q^{*}\end{array}\right\},$$ where $Q^*$ is the restriction of $Q$ to nonempty disjunctions.

\paragraph{Reformulating as a disjunctive program.}
Using \cref{lemma-l1}, we can rewrite \eqref{eqn-opt-Fw}, where $X(\Wset)$ denotes the set of extreme points of polytope $\Wset$: 
	$$ \max 
\braces{  \E
	[w(s) \staterewards(s) ]   \colon \;\; \forall \lambda \in \{-1,1\}^{\nS}, \;\; \exists \;\; g \in X(\tilde\Wset) \text{ s.t. } \sum_k \lambda_k m_k(w,g) = 0
	}$$
and therefore in the disjunctive syntax, 
	$$ \max 
\braces{  \E
	[w(s) \staterewards(s) ]   \colon \;\; \underset{\lambda \in\{-1,1\}^{\nS}}{\wedge} \;\; 
	\left(
	\underset{ 
 g \in X(\tilde\Wset) 
}{\vee} \text{ s.t. }  \lambda_k m_k(w,g) = 0
	\right)
}$$
Next, by applying the distributive property of conjunctions and disjunctions (\cite{balas1998disjunctive}), we can express the conjunctive form into the disjunctive form, and then apply Theorem 2.1 to obtain the corresponding linear programming representation. Note that this operation generates exponentially many unions of disjunctions (each also exponential in cardinality of state space), and therefore admits a overall superexponentially-sized program. 
	\end{proof}

	\begin{proof}[Proof of \Cref{lemma-l1}]
		For brevity, denote
\begin{align*} P(\lambda) &\defeq	\min
		\{
		\sum_j	 \lambda_j m_k(w, g)
		\colon  g \in \tilde \Wset
		\}
 \\
P(\ell_1) &\defeq
		\min\{ \sum_j \abs{m_k(w, g)} \colon g \in \tilde \Wset\}
\end{align*} 
so that the lemma can be stated as: 
$$ 	P(\lambda)	 = 0, \forall \lambda \in \{-1,1\}^p
\iff P(\ell_1) = 0  $$

	First we argue $\impliedby$:
	Suppose not,  where $\{  \min \sum_j \abs{m_k(w, g)} = 0, \forall j  \}   $ is true (and achieved by some $w^*$)
	but there is some $\lambda^*$
	such that $\min\{ \sum_j \lambda_j^*  	m_k(w, \lambda)
	\colon a \leq w \leq b \} > 0$. 
	But $g^*$ was feasible for $P(\lambda^*)$, and $P(\ell_1) \leq 0$ implies that $m_k(w, g^*) = 0, \forall j$, so we could further reduce the value of $P(\lambda^*)$ at feasible $g^*$, contradicting optimality of $\lambda^*$ at a strictly positive value.
	
	In the other direction, $\implies$, suppose by way of contradiction that $P(\lambda)\leq 0,\;\; \forall \lambda$ but $P(\ell_1) > 0$ (equivalently, that $g$ is infeasible for $\Wset$). Let $\lambda^*_{\ell_1}$ be the particular sign pattern which achieves strict nonnegativity at optimality for $P(\ell_1)$. Then, the optimal value of $P(\lambda)$ at this particular choice of $\lambda^*_{\ell_1}$ is also strictly positive, $ P(\lambda^*_{\ell_1})  = P(\ell_1) > 0$ since it induces the same optimization objective over the same feasible set of weights, $\Wset$.
	
	By uniqueness of $w$ for feasible values of the unknown weights $g_k$ (by properties of ergodicity and uniqueness of the stationary distribution), it is not a relaxation to optimize over different weights for each $\lambda$ value.
\end{proof}

\begin{proof}[Proof of \Cref{lem-matrix-inversion-body}]
We recall the notation which encodes the estimating equation as the matrix $A$: we introduce the instrument functions
$\phi_s,\phi_{s'} \in \mathbb{R}^{\nS \times 1 }$, random (row) vectors which are one-hot indicators for the state random variable $s, s'$ taking on each value, $ \phi_s=  \begin{bmatrix}
\mathbb{I} [s=0] & \dots &  \mathbb{I} [s=\nS]
\end{bmatrix}$. Let $A(g) = \E[\phi_{s'}( {\pi_e(a\mid s)g_{s'}(a\mid s)} \phi_s - \phi_{s'})^\top]$ and $b_s=p_b^{\infty}(s)$. 
The set of $\Wpimarg \in \tilde\Wset$ that admit a feasible solution to the estimating equation for some $\wstat \in \Theta$ is 
$\psi\defeq \textstyle\{  g \in \tilde\Wset \colon  
\;\exists \; w\geq 0
\text{ s.t. } 
A(g) w = 0, b^\top w= 1  \}$

	$A$ has rank $\nS -1$ \textit{if} $g$ is feasible since satisfying the conformability constraint \Cref{eqn-Q-conformability-transitions,eqn-sharp-Q-flow} implies linear dependence on the rows of $A$:
	$$\sum_k ( w_k - \sum_{j,a}\Pstat_b (a,j \mid k) \pi_e(a\mid j) g_k(a \mid j ) w_j) =  0 $$

Define $\tilde A(g)$ by replacing the last row of $A(g)$ by $b$ and let $v=(0,\dots,0,1)$. Then $w = \tilde A^{-1} v$ which results in the following program, with $v =\begin{bmatrix}
	0_{\nS-1}& 1
	\end{bmatrix}^\top $: 
	$$ \inf ~/ ~\sup \left\{   \varphi^\top \tilde A^{-1} v
~~\colon~~ g\in\psi 
	\right\} $$
	Partial derivatives of the matrix-valued function of $g_k(a\mid j)$ follow from the matrix chain rule, where $J_{j,k}$ is a one-hot matrix with a 1 in the $j,k$ entry and 0 everywhere else:
	\begin{align*}
	\frac{ \partial \varphi^\top  \tilde A^{-1} v}{\partial g_k(a\mid j)} &{=} \sum_{i,j} \frac{\partial \varphi^\top  \tilde A^{-1} v }{ \partial \tilde A_{i,j} } \frac{\partial \tilde A_{i,j} }{\partial g_{k}(a\mid j) } \\
	& = -\sum_{i,j} \tilde A^{-\top }_{i,j} \varphi v^\top  \tilde A^{-\top }_{i,j}( 	\mathbb{I}[j \neq \nS]  \pi^e_{a,j} \Pstatcompactb_{j,a,k}  J_{j,k}  )\\
	& = -( 	\mathbb{I}[j \neq \nS]  \pi^e_{a,j} \Pstatcompactb_{j,a,k}  )( \tilde A^{-\top } \varphi w^\top)_{i,j}
	\end{align*} 
\end{proof}

	\section{Proof of \Cref{thm-consistency} (statistical consistency)}\label{apx-consistency}
	Consistency follows from stability of the optimization problem in terms of the deviations of empirical probabilities from their population values. The former is a result from variational analysis/stability analysis of linear programs, which is non-standard because the perturbations occur in the constraint matrix \textit{coefficients}. The latter is simply the convergence of empirical probabilities.

We prove statistical consistency by considering a support function estimate that discretizes the space (restricting attention to feasible $w$). We apply a stability analysis argument to argue consistency for every fixed $w$ in the discretization. Since $w$ is in a compact set (the $\Ss$ simplex), a covering argument over the solution space provides a bound via a union bound over elements of the discretization. Lastly, we bound the approximation error arising from the discretization. While the discretization approach provides a statistical consistency result (in the limit as $n\to \infty$ we assume the discretization grows finer), it is a tool of the analysis but not the algorithmic proposal.

\paragraph{Preliminaries: Stability Analysis}
Consistency follows from stability of the optimization problem in terms of the deviations of empirical probabilities from their population values. The former is a result from variational analysis/stability analysis of linear programs, which is non-standard because the perturbations occur in the constraint matrix \textit{coefficients}. The latter is simply the convergence of empirical probabilities.

%

Stability analysis establishes convergence in Hausdorff distance between $\hat \Theta$, $\Theta$, the partial identification set obtained from optimizing the sample estimating equation vs. from optimizing the population estimating equation. 
The Hausdorff distance between two sets $A,B \subset \mathbb R^d$ is $ 
{d}_{H}(A, B)=\max \left\{\sup_{a \in A} {d}(a, B), \sup _{b \in B} {d}(b, A)\right\}
$
where ${d}(a, B) = \inf_{b \in  B} \norm{ a - b} $; e.g. it measures the furthest distance from an arbitrary point in one of the sets to its closest neighbor in the other set. 

%

The main stability analysis result we use is Theorem 1 of \cite{robinson1975stability}. To help keep the presentation of the theorem self-contained, we state some preliminary notation. The paper considers the general case of a system of linear inequalities, 
where $A$ is a continuous linear operator from $X$ into $Y$ which are real Banach spaces, and $K$ is a nonempty closed convex cone in $Y$. We study 
\begin{equation}\label{eqn-unperturbed-system}Ax \leq_K b, \forall x \in C
\end{equation} with $C \subseteq X$ a convenience set to represent \textit{unperturbed} constraints. We want to ascertain the
\textit{stability region} of the solution set $G$, which implies that for \textit{each} $x_0 \in G$, for some positive number $\beta$, and for any continuous linear operator $A' \colon X \mapsto Y$ and any $b' \in Y$, the distance from $x_0$ to the solution set of the perturbed system, 
\begin{equation}\label{eqn-perturbed-system}A'x \leq_K b',
\end{equation}  is bounded by $\beta \rho(x_0)$, with $\rho(x)$ being the residual vector,
$$\rho(x) \defeq d (b' - A'x, K) \defeq \inf \{ \norm{b' - A'x - k } \mid k \in K \}.$$

For a more concise statement of the main stability analysis result, we introduce the augmented operator with an auxiliary dimension to homogenize the system $Q\colon X \times \mathbb R \mapsto Y $: for finite-dimensional systems of linear equations, this is the usual homogenization.  
$$ Q (\begin{bmatrix}
x \\  \xi
\end{bmatrix})
= \begin{cases}
\begin{bmatrix} A & -b \end{bmatrix} 
\begin{bmatrix}
x \\ \xi
\end{bmatrix}
+ K 
&, \begin{bmatrix}
x & \xi 
\end{bmatrix}^\top  \in P,
\\ \infty & \begin{bmatrix}
x & \xi 
\end{bmatrix}^\top,  \not\in P
\end{cases} $$
Now, under this notation, $x \in C$ satisfies $Ax \leq_K b, \forall x \in C$ iff $ 0 \in Q(\begin{bmatrix}
x \\  \xi
\end{bmatrix})$.
The result will leverage properties of linear operators as special cases of convex processes (which are themselves multivalued functions between two linear spaces): If $T$ carries $X$ into $Y$, with $X,Y$ normed linear spaces, then the \textit{inverse} of $T$ is $T^{-1}$, which is defined for $y \in \mathcal Y$ by $$T^{-1} y \defeq \{ x \mid  y \in T x \}.$$ 
$T$ is closed if $\op{gph}(T) \defeq \{ (x,y)  \mid y \in Tx \} $ is closed on product space $X \times Y$. The norm of $T$ is operator norm. 

This approach then allows us to identify another linear operator which parametrizes the \textit{perturbation} $\Delta ( \begin{bmatrix}
x \\ \xi
\end{bmatrix})$ defined analogously to $Q$ but with $(A'-A)x - (b'-b)\xi + K$. The size of the perturbation is measured by the operator norm of this system, and a crude bound is $\norm{ \Delta} \leq \norm{A' - A} + \norm{b' - b}$; we also have that $\rho(x) \leq \norm{\Delta} \max\{ 1, \norm{x} \}$.

\begin{assumption}
	Regularity: $b \in \op{int} \{ A(C) + K\}$ and singular otherwise. 
\end{assumption}
The required regularity assumption is similar to strict consistency of \cite{rockafellar1970convex}, which states that $0 \in \op{int}(\op{dom})(G)$, e.g. there exists $u,v$ such that all inequality constraints $f < 0$ hold strictly. 

Finally, having introduced the homogenized system $Q$ and the perturbation $\Delta$, we state the required theorem: $Q' = Q + \Delta$ is the perturbed augmented system. 
\begin{theorem}[Linear system stability (Theorem 1, \cite{robinson1975stability})]\label{thm-lp-stability}
	Suppose that the system \cref{eqn-unperturbed-system} is regular. Then $Q$  { is surjective, } $ \left\|Q^{-1}\right\|<+\infty,$ and if $\left\|Q^{-1}\right\|\|\Delta\|<1$, then is also regular (hence solvable), with $\left\|Q^{\prime-1}\right\| \leqq\left\|Q^{-1}\right\| /\left(1-\left\|Q^{-1}\right\|\|\Delta\|\right)$ . Further, if $G^{\prime}$ denotes the solution set of \cref{eqn-perturbed-system}, then for any $ x \in C$ with $\left\|Q^{\prime-1}\right\| \rho(x)<1$ we have 
	\begin{equation}
	\qquad d\left(x, G^{\prime}\right) \leqq\left[\frac{\| Q^{\prime-1}\|\rho(x) }{1-\left\|Q^{\prime-1}\right\| \rho(x)}\right](1+\|x\|)
	\end{equation}
\end{theorem}

 We now use these results to prove consistency. 
 \proof{Proof of \Cref{thm-consistency}}

We consider the proposed support function estimator which conducts a grid search over an $\epsilon$-covering, $\mathcal{E}_w$ covering 
${\mathbb R}^{\nS}_+$ such that $\mathcal{E}_w$
is the smallest set satisfying that $\min_{w' \colon }\norm{w - w'}_1 \leq \epsilon_{}\;\forall w \text{ s.t. } w^\top \vec{1}$. For simplicity, we consider the self-normalized version of the estimator that constrains $ w^\top \vec{1}  = 1$: convergence for the stochastic constraint $ \E[w] =1$ holds by additionally conditioning on the event $\norm{p_b^\infty (s) }_1 \leq \epsilon$ under fast geometric convergence of the stationary distribution.

The discretization-based support function estimator is: 
\begin{align*}\hat{\overline{R}}_e &= \max 
\braces{ \E_n[ w(s) \Phi(s) ] \;\; \colon  w \in \mathcal E_w , \; w\in \hat\Theta ,w^\top \vec{1}  =1 }\\
\text{ where }
\hat\Theta &\defeq \{ w \in \mathcal E_w \colon \sum_{j,a}
\hat{h}_{j,a,k}( w) \Wpimarg_k(a\mid j)  -w(k)  = 0,\; \forall k \in \Ss; \Wpimarg \in \tilde\Wset \}\\
\hat{h}_{j,a,k}( w) &\defeq \hat{p}^{(\infty)}_{b} ( j, a  \mid k  )\wstat(j) \pi_e (a \mid j).
\end{align*} 

We will bound the approximation error 
$$ \norm{\hat{\overline{R}}_e  - \overline{R}_e } \leq \norm{ \hat{\overline{R}}_e - \hat{\tilde{\overline{R}}}_e} +  \norm{  \hat{\tilde{\overline{R}}}_e - \tilde{\overline{R}}_e}
+ \norm{
	\tilde{\overline{R}}_e-
{\overline{R}}_e}
$$
with respect to the intermediary terms,
\begin{align*}
\hat{\tilde{\overline{R}}}_e &= \max 
\braces{ \E_n[ w(s) \Phi(s) ] \;\; \colon  w \in \mathcal E_w , \; w\in \Theta ,w^\top \vec{1}  =1 }\\
\tilde{\overline{R}}_e &= \max 
\braces{ \E[ w(s) \Phi(s) ] \;\; \colon \; w \in \mathcal E_w, w\in \Theta ,w^\top \vec{1}  =1 }\\
{\overline{R}}_e &= \max 
\braces{ \E[ w(s) \Phi(s) ] \;\; \colon \; w\in \Theta ,w^\top \vec{1}  =1 }\\
\end{align*} 
Bounding $\norm{ \hat{\overline{R}}_e - \hat{\tilde{R}}_e}$ follows by verifying stability of the feasibility problem $\{ \sum_{j,a}
\hat{h}_{j,a,k}( w) \Wpimarg_k(a\mid j)  -w(k)  = 0,\; \forall k \in \Ss\}$ for every such value of $w$, taking a union bound over the covering, and then applying stability of $w$ for a given feasible $g$ to bound the objective values.
%


$G(w), \hat G(w)$ correspondingly denote the solution sets (in the space of $g$) of the feasibility program, for a fixed $w$ vector. We first bound $\norm{ \hat{\overline{R}}_e - \hat{\tilde{\overline{R}}}_e}$ by bounding the distance between the solution set for $G(w), \hat G(w)$ for a fixed $w$. We then apply the bound for every $w$ on a covering of the $\nS$ simplex.

To map the problem quantities to the stability analysis notation, let the set of unperturbed constraints be $C \defeq \{ x \in \Wset \}$.
Observe that $\norm{Q'^{-1}} = \max_{\norm{y} \leq 1} \{ \norm{x} \colon  Q' x = y , x \in C \} \leq \nS^2 \nA \nu$ because the \textit{unperturbed} constraints include bounds constraints on $g$.

Bounding $\norm{\Delta}$ proceeds by observing that since the perturbation matrix to the coefficients comprises of terms $$
\sum_{j,a} w(j) \pi_e(a\mid j)  (\hat p^\infty_b (j, a \mid k' ) -  p^\infty_b (j, a \mid k' )),$$ 
applying
a crude bound that $w(j), \pi_e(a \mid j) \leq 1$,
it is sufficient to bound the operator norm of perturbations as $$\norm{\Delta}_1 \leq  \norm{ \hat{p}^\infty_b(s, a \mid s' )- \Pstat_b(s, a \mid s' )}_\infty.$$ 

Consistency of the empirical state-action probabilities for the population state-action probabilities yields the result, e.g. that $\hat p^\infty_b (s, a , s' )  \to_p  p^\infty_b (s,a,s' )$, and $\hat p^\infty_b (s)  \to_p  p^\infty_b (s)$. (See \cite{mannor2005empirical} for quantitative rates). 
Consider $n$ large enough so that the condition $\left\|Q^{\prime-1}\right\| \rho(x)<\frac 12$ holds. 
Then,
\begin{equation}\label{bnd-fixw} d(\hat x, G) \leq
\left[\frac{\| Q^{\prime-1}\|\rho(x) }{1-\left\|Q^{\prime-1}\right\| \rho(x)}\right](1+\|x\|) \leq 2
 (1+\nS^2 \nA \nu)^2
\norm{
	\hat{p}^\infty_b(s, a \mid s' )  - 
	\Pstat_b(s, a \mid s' )}_\infty
\end{equation} for any $\hat x \in \hat G$, with probability 1; and analogously for $d( x, \hat G)$. The above bound holds for a fixed $w$.

We next bound $\textstyle \norm{ \hat{\overline{R}}_e - \hat{\tilde{\overline{R}}}_e}$. 
Taking a union bound over finitely many $w \in \mathcal E_w$ so that the bound \eqref{bnd-fixw} holds uniformly over $\mathcal E_w$, there exists $n_1$ large enough such that 
\begin{equation}\label{lemma-gclose} \max_{w \in \mathcal E_w} d(\hat G(w), G(w)) \leq \delta.
\end{equation}
This in turn suggests finite-time identification of the feasible set, $\hat \Theta(\mathcal E_w) = \Theta(\mathcal E_w) $, so that $ \hat{\overline{R}}_e = \hat{\tilde{R}}_e $.

We next bound $ \norm{  \hat{\tilde{\overline{R}}}_e - \tilde{\overline{R}}_e}$. For every $\epsilon'$, by geometric convergence of the stationary distribution $\hat{p}^\infty_b(s)  - {p}^\infty_b(s)$ and taking a union bound over $w \in \mathcal{E}_w$ to establish uniform convergence, there exists $n_2$ large enough so that:
\begin{align*}
 \norm{  \hat{\tilde{\overline{R}}}_e - \tilde{\overline{R}}_e} \leq \max_{w \in \Theta} \{ \E_n[ w(s) \Phi(s) ] - \E[w(s) \Phi(s)] \} \leq \epsilon'.
\end{align*} 
Lastly, $\norm{
	\tilde{\overline{R}}_e-
	{\overline{R}}_e}$ is bounded by the uniform approximation error which is satisfied by the definition of the covering, and bounded state rewards $\Phi(s)$. Therefore, we have that for some $n \geq\max\{ n_1, n_2\}$,
$$ \norm{\hat{\overline{R}}_e  - R_e } \leq \epsilon' + \epsilon
$$
Therefore we obtain statistical consistency as we take the discretization width $\epsilon \to 0$ as $n\to \infty$.

\endproof

\subsection{Linear function approximation}\label{apx-linearfunctionapprox}

Note that since the tabular case is a special case of linear function approximation of the state space, our approach also handles the case where $w = \theta^\top s$ is a linear parameter of the state observation. 

We introduce some simplifications. 
$\Psi_k^{i: t,t+1} =  ( \phi(s_{t+1}^i ) \phi(s_t^i)^\top)_k$ is the $k$th row vector (resp. $\Psi_k^{t,t}$) and 
\begin{align*}
\bar{\phi}(s) &= \frac{1}{NT}\sum_{i=1}^N \sum_{t=1}^T \phi(s_{t}^i) = \E_N\E_T[\phi(s)]\\
\bar{\Psi}^{ t+1,t+1} &= \frac{1}{NT} 
\sum_{i=1}^N \sum_{t=1}^T   \phi(s_{t+1}^i ) \phi(s_t^i)^\top
\end{align*}
The corresponding finite-sample feasibility oracle for the case $w = \theta^\top s$, drawing on the least-squares representation of \cref{lem-matrix-inversion-body}, is: 
\begin{align*}
F(\theta) \defeq	\min_{g \in \Wset,z\geq 0}\;& 
\sum_k z_k  \\
\st &	z_k\geq \frac{1}{NT} 
\sum_{i=1}^N \sum_{t=1}^T 
\left(\pi^e_{i,t} 
\Psi_k^{i: t,t+1} \theta 
\cdot g_{i,t}
- \Psi_k^{i: t+1,t+1}\theta 
\right),
, \;&& \forall k \in [d]\\
&	z_k\geq -\frac{1}{NT} 
\sum_{i=1}^N \sum_{t=1}^T 
\left(\pi^e_{i,t}
\Psi_k^{i:t,t+1} \theta 
\cdot g_{i,t}
- \Psi_k^{i: t+1,t+1}\theta 
\right) 
, \;&&  \forall k \in [d]\\
&   l_{i,t} \leq g_{i,t}\leq m_{i,t} , \;\; \forall i \in [N], t \in [T] \\
&
\frac{1}{NT} \sum_{i=1}^N \sum_{t=1}^T \mathbb{I}[A_{it} = a']   g_{i,t} = 1, \forall a'\\ 
\end{align*}
To obtain bounds, using $\E_N \E_T$ as shorthand notation for empirical expectations over trajectories and timesteps within trajectories, one then solves: 
\begin{equation}\label{eqn-opt-Fw-linear}
\inf/\sup~\{ \E_N\E_T[\Phi(s)]^\top \theta  \colon F(w) \leq 0,~ \E_N\E_T[\phi(s)]^\top \theta  \}.
\end{equation}
since $\E_N\E_T[w(s)] = \E_N\E_T [ \phi(s)^\top\top \theta ] = \E_N\E_T[\phi(s)]^\top \theta $.

Finally we remark on additional difficulties in the linear function approximation setting that are not encountered in the tabular setting. The computational burden increases because the reparametrization argument of \cref{lemma-weight-reparametrization} crucially relied on the discrete distribution of $S$. In practice, introducing a number of bilinear variables which grows with $n$ performs quite poorly in comparison to the re-parametrization. Furthermore, the approach \textit{requires} well-specification of the linear function class for $w$, since we effectively require realizability of a linear parameter for some set of possible inverse propensity weights. Handling model misspecification via a feasibility relaxation for the approximation error of the function class may therefore tradeoff sharpness; we leave this for future work.  

\section{Additional Empirics and Details}\label{apx-empirics}

\paragraph{Confounded gridworld}
In \Cref{fig-gridworld} we include a depiction of the 3x3 gridworld as described in the main text describing the reward structure in greater detail. (If an action moves an agent into a wall, it simply remains in place). 

\paragraph{SDP relaxation}
One possibility is to consider the standard SDP relaxation for nonconvex quadratic programs \cite{optimization2004s}. Let $x = \begin{bmatrix}
g \\ w
\end{bmatrix}$, and $P_k$ be the matrix that generates the quadratic form $ \sum_{a,j} w(j) \Pstat_b(s,a,s) g_k(a\mid j) \pi_e ( a \mid j)$, e.g. we have that $P^k_{(k,a,j), \nS^2   \nA+j } = \pi_e(a \mid j) \Pstat_b(s,a,s), \forall a,j$. Then a standard relaxation gives that the solution to the following semidefinite program is a lower bound: 
\begin{align*}
\max\braces{ \E[ w(s) \Phi(s)] \colon \op{Tr}(X P_k ) - w(k)\Pstat_b(k) = 0, \forall k; \;\; g \in \tilde{\Wset}, \E_b[w(s)]=1, \begin{bmatrix}
	X & x \\x^\top & 0 
	\end{bmatrix}  \succeq 0}
\end{align*}
However, the lifting results in a $(\nS^2 \nA)$ square matrix; given that under mild assumptions, typical SDPs can be solved in $O(n^2 m^{5/2} \log(\epsilon^{-1}))$ where $n$ is dimension of the vector variable and $m$ is dimension of the matrix variable, we incur an intractable $O(\nS^9 \nA^{9/2})$ scaling overall.

\paragraph{Details for confounded random walk }
Note that the stationary distribution under $\pi_b$ is, for $s_1, s_2$: 
$$
\left( \frac{-1+(1-2\pi_{s_2u_1})(p_{u_1} + p_{u_2}) + \pi_{s_2u_1} }{-1 + (\pi_{s_2u_2} - \pi_{s_1u_1})(1-2p_{u_1} -2p_{u_2} )} , 
\frac{p_{u_1}+p_{u_2} + \pi_{s_1u_1} (1 - 2p_{u_1} - 2p_{u_2}) }{1 + (\pi_{s_1u_1} - \pi_{s_2 u_1}  )(1 - 2 p_{u_1} - 2p_{u_2})}\right)
$$

\paragraph{\cref{alg:pgd} vs. global optimization}
We compare the results from using \cref{alg:pgd} (nonconvex projected gradient descent) vs. solving to full optimality by Gurobi. We find that empirically, storing previous solutions and imposing monotonicity (e.g. for any $\Gamma$, taking the max of all previous returned values and the computed values) helps stabilize the optimization. 

\begin{figure}[t!]
	\begin{minipage}{0.3\textwidth}
	\includegraphics[width=0.8\textwidth]{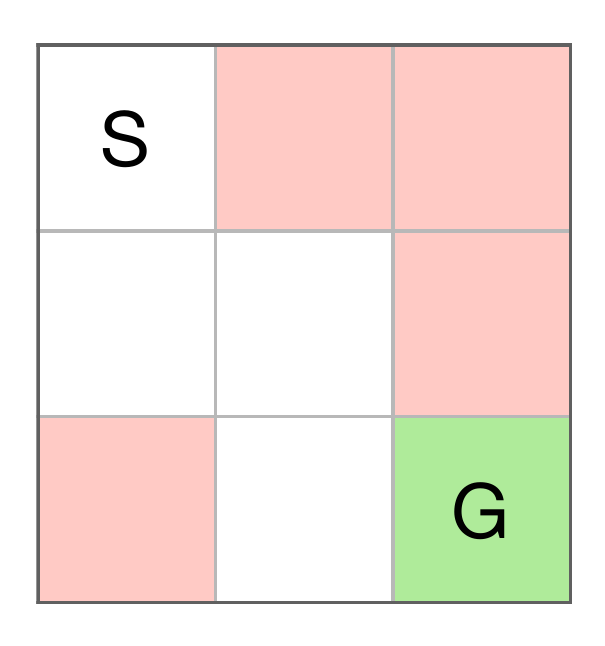}
	\caption{3x3 gridworld. \\$\Phi(s)$ of green is 1, red -0.3.}
	\label{fig-gridworld} 
\end{minipage}
\begin{minipage}{0.7\textwidth}
	\centering
	\includegraphics[width=0.8\textwidth]{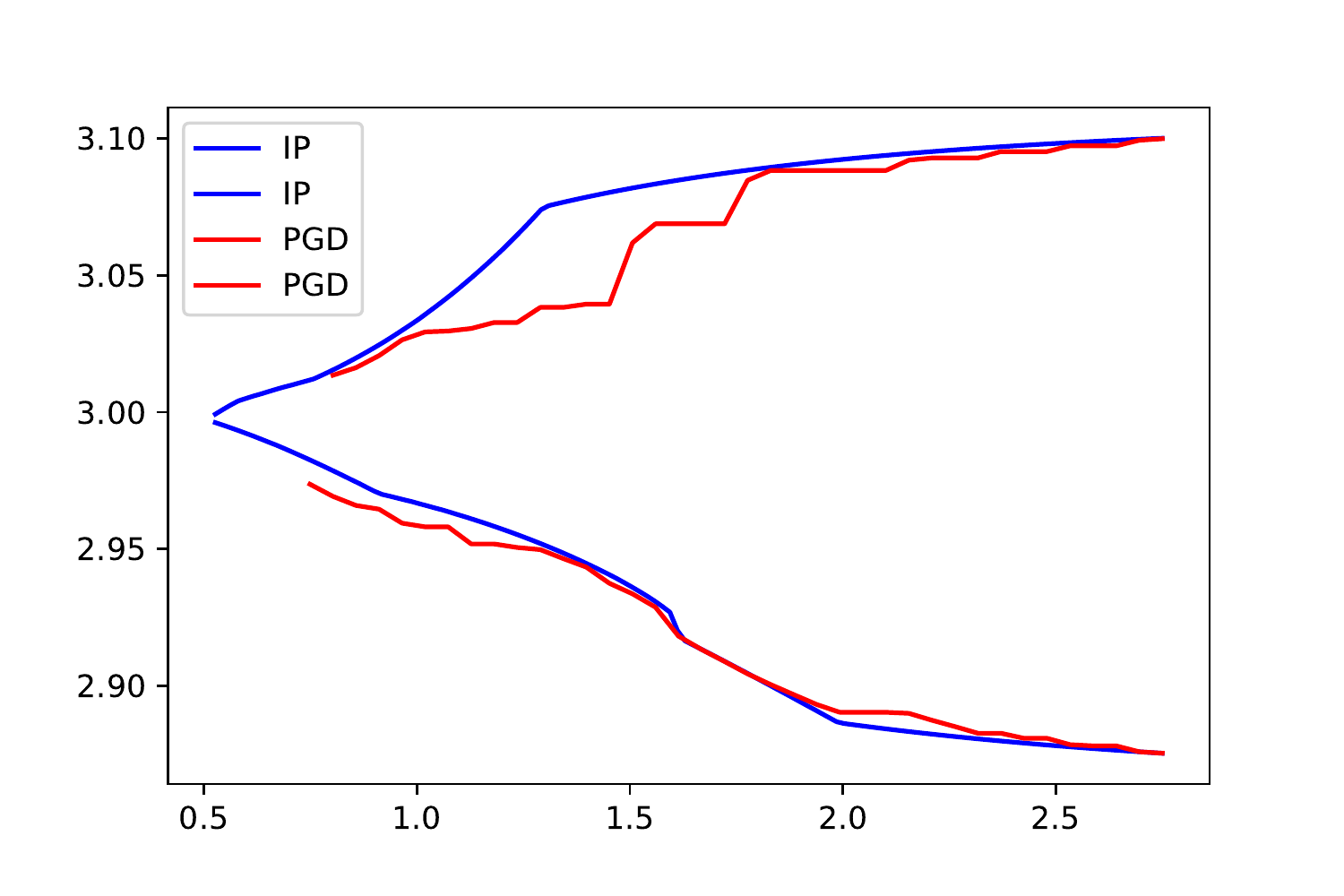}
	\caption{Comparison of global optimization and \cref{alg:pgd}.}
\end{minipage}
\end{figure}


\paragraph{Computational details}
A complete description of the data collection process, including sample size.
\begin{itemize}
\item Data: bounds computed based on a trajectory with 40000 steps, and a grid of 25 linearly spaced $\Gamma$ values from $\log(\Gamma) \in [0.1,1.7]$ (equivalently, $\Gamma \in [1.10, 5.47]$). 
\item Experiments were run on a Macbook Pro with 16gb RAM.
\item Packages: Python (numpy/scipy/pandas), Gurobi Version 9
\end{itemize}

\end{document}